\documentclass[11pt]{article}
\usepackage{times}
\usepackage[left=1.25in, top=1.25in, bottom=1.25in, right=1.25in]{geometry}

\usepackage{url}
\usepackage[colorlinks,
            linkcolor=red,
            anchorcolor=blue,
            citecolor=blue,
            backref=page
            ]{hyperref}       
\usepackage{url}            
\usepackage{booktabs}       
\usepackage{amsfonts}       
\usepackage{nicefrac}       
\usepackage{microtype}      
\usepackage{xcolor}         

\usepackage{amsmath, amssymb, amsthm, bm}
\usepackage{multirow}
\usepackage{natbib}
\usepackage{graphicx}
\usepackage{makecell}
\usepackage{booktabs}
\usepackage{array}
\usepackage{algorithm}
\usepackage{algorithmic}
\usepackage{dsfont}
\usepackage{csquotes}
\usepackage{cancel}
\usepackage{cleveref}
\usepackage{enumitem}
\usepackage{thmtools}
\usepackage{thm-restate}
\usepackage{caption}
\usepackage{subcaption}
\usepackage{nicefrac}
\usepackage{math_commands}
\usepackage{tikz}
\usetikzlibrary{arrows.meta}

\DeclareMathOperator{\arcsinh}{arcsinh}

\theoremstyle{plain}
\newtheorem{theorem}{Theorem}[section]
\newtheorem{lemma}[theorem]{Lemma}
\newtheorem{remark}[theorem]{Remark}

\newtheorem{corollary}[theorem]{Corollary}
\theoremstyle{definition}
\newtheorem{definition}[theorem]{Definition}

\newtheorem{conjecture}[theorem]{Conjecture}
\newtheorem{assumption}[theorem]{Assumption}

\newtheorem{example}[theorem]{Example}
\newtheorem{proposition}[theorem]{Proposition}
\crefname{claim}{claim}{claims}
\crefname{conjecture}{conjecture}{conjectures}
\crefname{assumption}{assumption}{assumptions}
\crefname{condition}{condition}{conditions}

\crefformat{thanks}{#2\footnotemark[#1]#3}

\DeclareMathOperator{\dom}{\mathrm{dom}}
\DeclareMathOperator{\range}{\mathrm{range}}
\newcommand{\xinit}{{x_{\mathrm{init}}}}
\newcommand{\winit}{{w_{\mathrm{init}}}}
\newcommand{\concate}{\ \|\ }
\newcommand{\len}{\mathrm{len}}
\newcommand{\Lie}{\mathrm{Lie}}

\makeatletter
\newlength\aftertitskip     \newlength\beforetitskip
\newlength\interauthorskip  \newlength\aftermaketitskip

\def\maketitle{\par
 \begingroup
   \def\thefootnote{\fnsymbol{footnote}}
   \def\@makefnmark{\hbox to 4pt{$^{\@thefnmark}$\hss}}
   \@maketitle \@thanks
 \endgroup
\setcounter{footnote}{0}
 \let\maketitle\relax \let\@maketitle\relax
 \gdef\@thanks{}\gdef\@author{}\gdef\@title{}\let\thanks\relax}

\def\@startauthor{\noindent \normalsize\bf}
\def\@endauthor{}
\def\@starteditor{\noindent \small {\bf Editor:~}}
\def\@endeditor{\normalsize}
\def\@maketitle{\vbox{\hsize\textwidth
 \linewidth\hsize \vskip \beforetitskip
 {\begin{center} \LARGE\@title \par \end{center}} \vskip \aftertitskip
 {\def\and{\unskip\enspace{\rm and}\enspace}%
  \def\addr{\small\it}%
  \def\email{\hfill\small\tt}%
  \def\name{\normalsize\bf}%
  \def\AND{\@endauthor\rm\hss \vskip \interauthorskip \@startauthor}
  \@startauthor \@author \@endauthor}
}}

\makeatother

\newcommand{\symfootnotetext}[1]{%
\let\oldthefootnote=\thefootnote%
\stepcounter{mpfootnote}%
\renewcommand{\thefootnote}{\fnsymbol{mpfootnote}}%
\footnotetext{#1}%
\let\thefootnote=\oldthefootnote%
}

\author{\name Zhiyuan Li \footnotemark[1] \email{zhiyuanli@cs.princeton.edu}\\
    \addr{Princeton University}\\ 
    \name Tianhao Wang \footnotemark[1] \email{tianhao.wang@yale.edu}\\
    \addr{Yale University}\\
    \name Jason D. Lee \email{jasonlee@princeton.edu}\\
    \addr{Princeton Unversity}\\
    \name Sanjeev Arora \email{arora@cs.princeton.edu}\\
    \addr{Princeton University}\\ 
}

\title{\textbf{Implicit Bias of Gradient Descent on Reparametrized Models:  On Equivalence to Mirror Descent}}

\ifdefined\usebigfont

\usepackage{times}
\usepackage[fontsize=13pt]{scrextend}
\AtBeginDocument{\newgeometry{letterpaper,left=1.56in,right=1.56in,top=1.71in,bottom=1.77in}}
\else
\fi
\ifdefined\usebigfont

\usepackage{times}
\usepackage[fontsize=13pt]{scrextend}
\AtBeginDocument{\newgeometry{letterpaper,left=1.56in,right=1.56in,top=1.71in,bottom=1.77in}}
\else
\fi
\newcommand{\jnote}[1]{\textcolor{blue}{}}

\begin{document}

\maketitle
\symfootnotetext{Equal contribution}
\begin{abstract}
As part of the effort to understand implicit bias of gradient descent in overparametrized models, several results have shown how the training trajectory on the overparametrized model can be understood as mirror descent on a different objective. 
The main result here is a  characterization of this phenomenon under a notion termed {\em commuting parametrization}, which encompasses all the previous results in this setting. 
It is shown that gradient flow with any commuting parametrization is equivalent to continuous mirror descent with a related Legendre function. 
Conversely,  continuous mirror descent with any Legendre function can be viewed as gradient flow with a related commuting parametrization. 
The latter result relies upon Nash's embedding theorem. 
\end{abstract}

\section{Introduction}\label{sec:intro}

\emph{Implicit bias} refers to the phenomenon in machine learning that the 
solution obtained from loss minimization has special properties that were not 
implied by value of the loss function and instead
arise from the trajectory taken in parameter space by the optimization.  
Quantifying implicit bias necessarily has to go beyond the traditional black-box 
convergence analyses of optimization algorithms. 
Implicit bias can explain how choice of optimization algorithm can affect 
generalization~\citep{woodworth2020kernel,li2022what,li2020towards}.

Many existing results about implicit bias treat training  (in the limit of 
infinitesimal step size) as a differential equation or process~$\{x(t)\}_{t\geq 0} \subset \RR^D$.
To show the implicit bias of $x(t)$, the idea is to show for another (more 
intuitive or better understood) process $\{w(t)\}_{t\geq 0} \subset \RR^d$ that 
$x(t)$ is \emph{simulating} $w(t)$, in the sense that there exists a mapping 
$G: \RR^D \to \RR^d$ such that $w(t) = G(x(t))$.
Then the implicit bias of $x(t)$ can be characterized by translating the special 
properties of $w(t)$ back to $x(t)$ through $G$. A related term, \emph{implicit regularization}, refers 
to a handful of such results where particular update rules are shown to lead to
regularized solutions; specifically, 
$x(t)$ is simulating $w(t)$ where $w(t)$ is solution to a regularized version of the original loss.

The current paper develops a general framework involving optimization in the 
continuous-time regime of a loss $L:\RR^d\to\RR$ that has been re-parametrized 
before optimization\footnote{Two examples from recent years, where $G$ does not 
change expressiveness of the model, involve (a) overparametrized linear 
regression where the parameter vector $w$ is reparametrized (for example as 
$w = u^{\odot 2} - v^{\odot 2}$~\citep{woodworth2020kernel}) and (b) deep linear 
nets~\citep{arora2019implicit} where a matrix $W$ is factorized as 
$W = W_1 W_2 \cdots W_L$ where each $W_\ell$ is the weight matrix for the 
$\ell$-th layer.} as $w = G(x)$ for some 
$G: \RR^D \to \RR^d$. 
Then the original loss $L(w)$ in the $w$-space induces the implied loss 
$(L \circ G)(x) \equiv L(G(x))$ in the $x$-space, and the gradient flow in the $x$-space is given by
\begin{align}\label{eq:gf_x}
    \diff x(t) = - \nabla (L \circ G)(x(t)) \diff t.
\end{align}
Using $w(t) = G(x(t))$ and the fact that $\nabla (L\circ G)(x) = 
\partial G(x)^\top \nabla L(G(x))$ where $\partial G(x) \in \RR^{d \times D}$ 
denotes the Jacobian of $G$ at $x$, 
the corresponding dynamics of~\eqref{eq:gf_x} in the $w$-space is
\begin{align}\label{eq:gf_w}
    \diff w(t) = \partial G(x(t)) \diff x(t)
    = - \partial G(x(t)) \partial G(x(t))^\top \nabla L(w(t)) \diff t.
\end{align}

Our framework is developed to fully understand phenomena in recent 
papers~\citep{gunasekar2018implicit, vaskevicius2019implicit,yun2020unifying, 
amid2020winnowing, woodworth2020kernel,amid2020reparameterizing,azulay2021implicit},
which give examples suggesting that gradient flow in the $x$-space could 
end up simulating a more classical algorithm, mirror descent (specifically, the 
continuous analog, mirror flow) in the $w$-space. 
Recall that mirror flow is continuous-time limit of the classical mirror 
descent, written as $\diff \nabla R(w(t)) = - \nabla L(w(t)) \diff t$ 
where $R: \RR^d \to \RR \cup \{\infty\}$ is a strictly convex function~\citep{nemirovskij1983problem, beck2003mirror}, which is called \emph{mirror map} or \emph{Lengendre function} in literature.
Equivalently it is {\em Riemannian gradient flow} with metric tensor $\nabla^2 R$, 
an old notion in geometry:
\begin{align}\label{eq:riemannian_gradient_flow}
    \diff w(t) = - \nabla^2 R(w(t))^{-1} \nabla L(w(t)) \diff t.
\end{align}
If there exists a Legendre function $R$ such that 
$\partial G(x(t)) \partial G(x(t))^\top = \nabla^2 R(w(t))^{-1}$ for all $t$, 
then~\eqref{eq:gf_w} becomes a simple mirror flow in the $w$-space. 
Many existing results about implicit bias indeed concern reparametrizations $G$ 
that satisfy $\partial G(x) \partial G(x)^\top = \nabla^2 R(w)^{-1}$ for a 
strictly convex function $R$, and the implicit bias/regularization is demonstrated by showing that the convergence point satisfies
the KKT conditions needed for  minimizing $R$ among all minimizers of the loss $L$.
A concrete example is that $w_i(t)= G_i(x(t))= (x_i(t))^2$ for all $i\in[d]$, so here $D=d$. 
In this case, the Legendre function $R$ must satisfy $(\nabla^2 R(w(t)))^{-1} = \partial G(x(t)) \partial G(x(t))^\top = 4\diag((x_1(t))^2,\ldots, (x_d(t))^2) = 4\diag(w_1(t),\ldots, w_d(t))$ which suggests $R$ is the classical negative entropy function, \emph{i.e.}, $R(w) = \sum_{i=1}^d w_i(\ln w_i -1)$. 

However, in general, it is hard to decide \emph{whether gradient flow for a given parametrization $G$ can be written as mirror flow for some Legendre function $R$}, especially when $D>d$ and $G$ is not an injective map. 
In such cases, there could be multiple $x$'s mapping to the same $G(x)$ yet having different $\partial G(x) \partial G(x)^\top$. 
If more than one of such $x$ can be reached by gradient flow, then the desired Legendre function cannot exist. \footnote{To avoid such an issue, \citet{amid2020reparameterizing} has to assume all the preimages of $G$ at $w$ have the same  $\partial G (\partial G)^\top$ and  a recent paper \citet{ghai2022non} assumes that $G$ is injective. }
If only one of such $x$ can be reached by gradient flow, we must decide which $x$ it is in order to decide the value of $\nabla^2 R$ using $\partial G \partial G^\top$. 
Conversely, \citet{amid2020reparameterizing} raises the following question:  {\em for what Legendre function $R$ can the corresponding mirror flow be the result of gradient flow after some reparametrization $G$?}
Answering the questions in both directions requires a deeper understanding of the impact of parametrizations.

The  following are the main contributions of the current paper:
\begin{itemize}
    \item In \Cref{sec:CGF_to_MF}, building on classic study of commuting vector fields we identify a 
    notion of when a parametrization $w = G(x)$ is {\em commuting} (\Cref{def:commuting_param}) and use it to
    give a sufficient condition (\Cref{thm:commuting_to_mirror}) and a slightly weaker necessary condition (\Cref{thm:necessary_condition}) of when 
    the gradient flow in the $x$-space governed by $-\nabla (L \circ G)$ is simulating a mirror flow 
    in the $w$-space with respect to some Legendre function $R:\RR^d\to\RR$, which encompasses all the previous results~\citep{gunasekar2018implicit, vaskevicius2019implicit,yun2020unifying, amid2020winnowing, woodworth2020kernel,amid2020reparameterizing,azulay2021implicit}.
    Moreover, the Legendre function is independent of the loss $L$ and depends only on 
    the initialization $\xinit$ and the parametrization~$G$.

    \item  We recover and generalize existing implicit regularization results for underdetermined linear regression as implications of the above characterization (\Cref{cor:overparam_linear_model}).  
    We also give new convergence analysis in such settings (\Cref{cor:convergence_commuting_quadratic_parametrization}), filling the gap in previous works~\citep{gunasekar2018implicit, woodworth2020kernel,azulay2021implicit} where parameter convergence is only assumed but not proved. 
    \item In the reverse direction, we use the famous Nash embedding theorem to show that every mirror flow in the $w$-space with respect to some Legendre function $R$ simulates a gradient flow with commuting parametrization under some embedding $x = F(w)$ where $F:\RR^d\to\RR^D$ and the parametrization $G$ is the inverse of $F$ (\Cref{thm:mirror_to_commuting}).
    This provides an affirmative and fully general answer to the question of 
    when such reparametrization functions exist, giving a full answer to questions raised in a more 
    restricted setting in~\citet{amid2020reparameterizing}.
\end{itemize}

\section{Related work}

\paragraph{Implicit bias.}
With high overparametrization as used in modern machine learning, there usually exist multiple optima, and it is crucial to understand which particular solutions are found by the optimization algorithm.
Implicit bias of gradient descent for 
classification tasks with separable data was studied in~\citet{soudry2018implicit,
gunasekar2018characterizing, nacson2019convergence, ji2021characterizing,
moroshko2020implicit,ji2020directional} 
and for non-separable data in~\citet{ji2018risk,ji2019implicit}, where the
implicit bias appears in the form of margin maximization.
The implicit bias for regression problems has also been analyzed by leveraging 
tools like mirror descent~\citep{woodworth2020kernel, gunasekar2018characterizing,
yun2020unifying, vaskevicius2019implicit, amid2020winnowing, amid2020reparameterizing},
later generalized in~\citet{azulay2021implicit}.

The sharp contrast between the so-called \emph{kernel} and \emph{rich} regimes~\citep{woodworth2020kernel} reflects the importance of the initialization scale, 
where a large initialization often leads to the kernel regime with features 
barely changing during training~\citep{jacot2018neural, chizat2018lazy,
du2018gradient, du2019gradient, allen2019convergence, allen2019learning, 
zou2020gradient, arora2019fine,yang2019scaling,jacot2021deep}, while with a small initialization, the solution
exhibits richer behavior with the resulting model having lower 
complexity~\citep{gunasekar2018implicitconv, gunasekar2018implicit, 
li2018algorithmic, razin2020implicit, arora2019implicit, chizat2020implicit,li2020towards,lyu2019gradient,lyu2021gradient,razin2022implicit,stoger2021small,ge2021understanding}. 
Recently \citet{yang2021tensor} give a complete characterization on the relationship between initialization scale, parametrization and learning rate in order to avoid kernel regime.

There are also papers on the implicit bias of other types of optimization 
algorithms, e.g., stochastic gradient descent~\citep{li2019towards,blanc2020implicit,
haochen2020shape, li2022what, damian2021label, zou2021benefits} and adaptive 
and momentum-based methods~\citep{qian2019implicit, wang2021momentum, 
wang2021implicit, ji2021fast}, to name a few.

\paragraph{Understanding mirror descent.}
In the continuous-time limit as step size goes to 0, the mirror flow is 
equivalent to the Riemannian gradient flow. 
\citet{gunasekar2021mirrorless}
showed that a partial discretization of the latter gives rise to the classical 
mirror descent. 
Assuming the existence of some reparametrization function, \citet{amid2020reparameterizing} 
showed that a particular mirror flow can be reparametrized as a gradient flow.
Our paper shows that such reparametrization always exists by using Nash's
embedding theorem. \citet{ghai2022non} generalizes the equivalence result of \citet{amid2020reparameterizing} to discrete updates.

\section{Preliminaries and notations}\label{sec:preliminary}

\textbf{Notations.} 
We denote $\NN$ as the set of natural numbers.
For any positive integer $n$, we denote $\{1, 2, \ldots, n\}$ by $[n]$.
For any vector $u\in\RR^D$, we denote its $i$-th coordinate by ${u_i}$.
For any vector $u,v\in\RR^D$ and $\alpha\in\RR$, we define $u \odot v = 
(u_1 v_1, \ldots, u_D v_D)^\top$ and $u^{\odot \alpha} = ((u_1)^\alpha, \ldots,
(u_D)^\alpha)^\top$.
For any $k\in\mathbb{N}\cup\{\infty\}$, we say a function $f$ is $\cC^k$ if it is $k$ times continuously 
differentiable, and use $\mathcal{C}^k(M)$ to denote the set of all $C^k$ functions from  $M$ to $\RR$. 
We use $\circ$ to denote the composition of functions, \emph{e.g.}, $f\circ g(x) = f(g(x))$.
For any convex function $R: \RR^D \to \RR\cup\{\infty\}$, we denote its domain 
by $\dom R = \{w \in \RR^D \mid R(w) < \infty\}$.
For any set $S$, we denote its interior by $\inter(S)$ and its closure by $\overline S$.
\\

We assume that the model has parameter vector $w \in \RR^d$ and $\cC^1$ loss function $L: \RR^d \to \RR$. 
Training involves a reparametrized vector $x \in \RR^D$, which is a reparametrization of $w$ such that $w = G(x)$ for some differentiable parametrization function $G$, and the objective is $L(G(x))$. 
From now on, we follow the convention that $d$ is the dimension of the original parameter $w$ and $D$ is the dimension of the reparametrized $x$.
We also refer to $\RR^d$ as the $w$-space and $\RR^D$ as the $x$-space.

In particular, we are interested in understanding the dynamics of gradient flow
under the objective $L\circ G$ on some submanifold $M\subseteq \RR^D$.
Most of our results also generalize to the following notion of \emph{time-dependent} loss.

\begin{definition}[Time-dependent loss]\label{def:time_loss}
A time-dependent loss $L_t(w)$ is a function piecewise constant in time 
$t$ and continuously differentiable in $w \in \RR^d$, that is, there exists 
$k\in \NN$, $0 = t_1 < t_2 < \cdots < t_{k+1}=\infty$ and $\mathcal{C}^1$ loss 
functions $L^{(1)}, L^{(2)},\ldots,L^{(k)}$ such that for each $i \in [k]$ and 
all $t \in [t_i, t_{i+1})$,
\begin{align*}
    L_t(w) =L^{(i)}(w), \qquad \forall w \in \RR^d.  
\end{align*}
We denote the set of such time-dependent loss functions by $\cL$.
\end{definition}

\subsection{Manifold and vector field}
Vector fields are a natural way to formalize the
continuous-time gradient descent (a good reference is~\citet{lee2013smooth}).  
Let $M$ be any smooth submanifold of $\RR^D$. 
A \emph{vector field} 
$X$ on $M$ is a continuous map from $M$ to $\RR^D$ such that for any $x \in M$, 
$X(x)$ is in the tangent space of $M$ at $x$, which is denoted by $T_x(M)$. 
Formally, $T_x(M):=\{\frac{\diff\gamma}{\diff t} \big|_{t=0} \mid 
\forall \textrm{ smooth curves } \gamma: \RR\to M, \gamma(0) = x\}$.

\begin{definition}[Complete vector field; p.215, \citealt{lee2013smooth}]
\label{def:complete_vec_field}
Let $M$ be a smooth submanifold of $\RR^D$ and $X$ be a vector field on $M$. 
We say $X$ is a \emph{complete vector field} on $M$ if and only if for any initialization $\xinit\in M$, the differential equation $\diff x(t) = X(x(t)) \diff t$ has a  solution on $(-\infty,\infty)$ with $x(0)=\xinit$.
\end{definition}

When the smooth submanifold $M \subseteq \RR^D$ is equipped with a metric tensor $g$, 
we then have a Riemannian manifold $(M, g)$, where for each $x\in M$, 
$g_x:T_xM \times T_xM \to \RR$ is a positive definite bilinear form.
In particular, the standard Euclidean metric $\overline{g}$ corresponds to $\overline{g}_x(u,v)=u^\top v$ 
for each $x\in M$ and $u,v\in T_xM$, under which the length of any arc on $M$ is
given by its length as a curve in $\RR^D$.

For any differentiable function $f:M \to \RR$, we denote by $\nabla_g f$ its
gradient vector field with respect to metric tensor $g$. 
More specifically, $\nabla_g f(x)$ is defined as the unique vector in $\RR^D$ such 
that $\nabla_g f(x)\in T_x(M)$ and 
$\frac{\diff f(\gamma(t))}{\diff t} \big\vert_{t=0} 
= g_x\big(\nabla f(x), \frac{\diff \gamma(t)}{\diff t} \big\vert_{t=0}\big)$. 
Throughout the paper, we assume by default that the metric on the submanifold 
$M\subseteq \RR^D$ is inherited from $(\RR^D, \overline{g})$, and we will use $\nabla f$ as a shorthand  for $\nabla_{\overline{g}} f$.
If $M$ is an open set of $\RR^D$, $\nabla f$ is then simply the ordinary gradient of $f$.

For any $x \in M$ and $\mathcal{C}^1$ function $f: M \to \RR$, we denote by 
$\phi_f^t(x)$ the point on $M$ reached after time $t$ by following the vector 
field $-\nabla f$ starting at $x$, \emph{i.e.}, the solution at time $t$ (when it exists) 
of 
\begin{align*}
\diff \phi^{t}_f = - \nabla f(\phi^{t} _f)\diff t,\qquad \phi^{0}_f(x) = x.
\end{align*}
We say $\phi_f^t(x)$ is \emph{well-defined} at time $t$ when the above differential 
equation has a solution at time $t$.
Moreover, for any differentiable function $X: M \to \RR^d$, we 
denote its Jacobian by 
\begin{align*}
\partial X(x) = (\nabla X_1(x), \nabla X_2(x), \ldots, \nabla X_d(x))^\top.
\end{align*}

\begin{definition}[Lie bracket]\label{defi:lie_bracket}
Let $M$ be a smooth submanifold of $\RR^D$.
Given two $C^1$ vector fields $X,Y$ on~$M$, we define the \emph{Lie Bracket} of 
$X$ and $Y$ as $[X,Y](x):= \partial Y(x) X(x) - \partial X(x) Y(x)$.
\end{definition}

\subsection{Parametrizations}
We use the term \emph{parametrization} to refer to differentiable maps from a 
smooth submanifold of $\RR^D$~($x$-space) to $\RR^d$~($w$-space).
We reserve $G$ to denote parametrizations, and omit the dependence on 
$G$ for notations of objects related to $G$ when it is clear from the context. 

The following notion of regular parametrization plays an important role 
in our analysis, and it is necessary for our main equivalence result between 
mirror flow and gradient flow with reparametrization. 
This is because if the null space of $\partial G(x)$ is non-trivial, \emph{i.e.,}
it contains some vector $u \neq 0$, then the
gradient flow with parametrization $G$ obviously cannot simulate any mirror flow 
with nonzero velocity in the direction of $u$. 

\begin{definition}[Regular parametrization]\label{def:regular_param}
Let $M$ be a smooth submanifold of $\RR^D$. 
A \emph{regular parametrization} 
$G: M \to \RR^d$ is a $\mathcal{C}^1$ parametrization such that 
$\partial G(x)$ is of rank $d$ for all $x \in M$.
\end{definition}

Note that a regular parametrization $G$ can become irregular when its domain is 
changed. 
For example, $G(x) = x^2$ is regular on $\RR^+$, but it is not regular on $\RR$ 
as $\partial G(0)=0$.

Given a $\mathcal{C}^2$ parametrization $G: M \to \RR^d$, for any $x \in M$ and 
$\mu \in \RR^d$, we define
\begin{align}\label{eq:flow_G}
 \psi(x; \mu) 
:= \phi_{G_1}^{\mu_1} \circ \phi_{G_2}^{\mu_2} \circ \cdots 
\circ \phi_{G_d}^{\mu_d}(x)
\end{align}
when it is well-defined, \emph{i.e.}, the corresponding integral equation has a solution.
For any $x \in M$, we define the domain of $\psi(x; \cdot)$ as
\begin{align}\label{eq:domain}
    \cU(x) = \big\{\mu \in \RR^d \mid \psi(x; \mu) \text{ is well-defined}\big\}.
\end{align}

When every $\nabla G_i$ is a complete vector field on $M$ as in 
\Cref{def:complete_vec_field}, we have $\cU(x) = \RR^d$.
However, such completeness assumption is relatively strong, and most polynomials would violate it. 
For example, consider $G(x) = x^{\odot 3}$ for $x \in \RR^d$, then the solution to 
$\diff x_i(t) = 3x_i(t)^2 \diff t$ explodes in finite time for each $i \in [d]$. 
To relax this, we consider parametrizations such that the domain of the flows 
induced by its gradient vector fields is pairwise symmetric.
More specifically, for any $x \in M$ and $i, j \in [d]$, we define 
\begin{align*}
\cU_{ij}(x) = \big\{(s, t) \in \RR^2 \mid \phi_{G_i}^s \circ \phi_{G_j}^t(x) 
\text{ is well-defined}\big\}, 
\end{align*}
and we make the following assumption.

\begin{assumption}\label{assump:domain}
Let $M$ be a smooth submanifold of $\RR^D$ and $G: M \to \RR^d$ be a parametrization.
We assume that for any $x \in M$ and $i \in [d]$, $\phi_x^t(x)$ is well-defined for $t \in (T_-, T_+)$ such that either $\lim_{t \to T_+} \|\phi_x^t(x)\|_2 = \infty$ or $T_+ = \infty$ and similarly for $T_-$.
Also, we assume that for any $x \in M$ and $i, j \in [d]$, it holds that 
$\cU_{ji}(x) = \{(t, s) \in \RR^2 \mid (s, t) \in \cU_{ij}(x)\}$, i.e., 
$\phi_{G_i}^s \circ \phi_{G_j}^t(x)$ is well-defined if and only if 
$\phi_{G_j}^t \circ \phi_{G_i}^s(x)$ does.
\end{assumption}

Indeed, under \Cref{assump:domain}, we can show that for any $x \in M$, $\cU(x)$ 
is a hyperrectangle, as summarized in the following lemma.
See \Cref{sec:apdx_pre} for a proof.

\begin{lemma}\label{lem:domain}
Let $M$ be a smooth submanifold of $\RR^D$ and $G: M \to \RR^d$ be a $\cC^2$
parametrization satisfying \Cref{assump:domain}.
Then for any $x \in M$, $\cU(x)$ is a hyperrectangle, i.e., $\cU(x)$ can be decomposed as
\begin{align*}
    \cU(x) = \cI_1(x) \times \cI_2(x) \times \cdots \times \cI_d(x)
\end{align*}
where $\cI_j(x):=\{x'_j\mid x'\in \cU(x)\}$ is an open interval.
\end{lemma}

For any initialization $\xinit\in M$, the set of points that are reachable via 
gradient flow under $G$ with respect to some time-dependent loss~(see \Cref{def:time_loss}) 
is a subset of $M$ that depends on $G$ and $\xinit$.

\begin{definition}[Reachable set]\label{def:reachable_set}
Let $M$ be a smooth submanifold of $\RR^D$.
For any $\mathcal{C}^2$ parametrization 
$G: M \to \RR^d$ and any initialization $\xinit \in M$, the reachable set 
$\Omega_x(\xinit;G)$ is defined as 
\begin{align*}
\Omega_x(\xinit;G) = \Big\{\phi_{L_1 \circ G}^{\mu_1} \circ \phi_{L_2 \circ G}^{\mu_2} 
\circ \cdots \circ \phi_{L_k \circ G}^{\mu_k}(\xinit)\  \Big|\  
\forall k \in \NN, \forall i\in[k], L_i \in \mathcal{C}^1(\RR^d), \mu_i\ge 0 \Big\}.
\end{align*}
\end{definition}

It is clear that the above definition induces a transitive ``reachable" 
relationship between points on $M$, and it is also reflexive since for all 
$L \in \mathcal{C}^1(\RR^d)$ and $t > 0$, $\phi^t_{L\circ G} \circ \phi^t_{(-L)\circ G}$ 
is the identity map on the domain of $\phi^t_{-L\circ G}$. 
In this sense, the reachable sets are orbits of the family of gradient vector 
fields $\{\nabla (L\circ G) \mid L \in \cC^1(\RR^d)\}$, \emph{i.e.}, the reachable 
sets divide the domain $M$ into equivalent classes.
The above reachable set in the $x$-space further induces the corresponding 
reachable set in the $w$-space given by $\Omega_w(\xinit; G) = G(\Omega_x(\xinit; G))$.

In most natural examples, the parametrization $G$ is smooth (though this is not 
necessary for our results), and by Sussman's Orbit Theorem~\citep{sussmann1973orbits}, 
each reachable set $\Omega_x(\xinit; G)$ is an immersed submanifold of $M$. 
Moreover, it follows that $\Omega_x(\xinit;G)$ can be generated by $\{\nabla G_i\}_{i=1}^d$, 
\emph{i.e.}, $\Omega_x(\xinit;G) = \{\phi_{G_{j_1}}^{\mu_1} \circ 
\phi_{G_{j_2}}^{\mu_2} \circ \cdots \circ \phi_{G_{j_k}}^{\mu_k}(\xinit) \mid 
\forall k \in \NN, \forall i\in[k], j_i\in [d], \mu_i\ge 0\}$.

\subsection{Mirror descent and mirror flow}
Next, we introduce some basic notions for mirror descent~\citep{nemirovskij1983problem, beck2003mirror}. 
We refer the readers to \Cref{apdx:convex_analysis} for more preliminaries on convex analysis.

\begin{definition}[Legendre function and mirror map]\label{def:legendre}
Let $R: \RR^d \to \RR \cup \{\infty\}$ be a differentiable convex function.
We say $R$ is a \emph{Legendre function} when the following holds:
\begin{enumerate}
\item[(a)]  $R$ is strictly convex on $\inter(\dom R)$.
\item[(b)]  For any sequence $\{w_i\}_{i = 1}^\infty$ going to the boundary of $\dom R$, 
$\lim_{i \to \infty} \|\nabla R(w_i)\|_2 = \infty$.
\end{enumerate}
In particular, we call $R$ a \emph{mirror map} if $R$ further satisfies the following condition (see p.298 in \citealt{bubeck2015convex}):
\begin{enumerate}
\item[(c)] The gradient map 
$\nabla R: \inter(\dom R) \to \RR^d$ is surjective. 
\end{enumerate}
\end{definition}

Given a Legendre function $R: \RR^d \to \RR\cup\{\infty\}$, for 
any initialization $w_0 = \winit \in \inter(\dom R)$, mirror descent with step 
size $\eta$ updates as follows:
\begin{align}\label{eq:MD}
    \nabla R(w_{k+1}) = \nabla R(w_k) - \eta \nabla L(w_k).
\end{align}
Usually $\nabla R$ is required to be surjective so that after a discrete descent 
step in the dual space, it can be projected back to the primal space via $(\nabla R)^{-1}$.
Nonetheless, as long as $\nabla R(w_k) -\eta \nabla L(w_k)$ is in the range of 
$\nabla R$, the above discrete update is well-defined.
In the limit of $\eta \to 0$, \eqref{eq:MD} becomes the continuous 
mirror flow:
\begin{align}\label{eq:MF}
    \diff \nabla R(w(t)) = - \nabla L(w(t)) \diff t.
\end{align}

Given a differentiable function $R$, the corresponding Bregman divergence $D_R$ is defined as
\begin{align*}
    D_R(w, w') = R(w) - R(w') - \langle \nabla R(w'), w - w'\rangle.
\end{align*}
We recall a well-known implicit bias result for mirror flow (which holds for 
mirror descent as well)~\citep{gunasekar2018characterizing}, which shows that 
for a specific type of loss, if mirror flow converges to some optimal solution, 
then the convergence point minimizes some convex regularizer among all 
optimal solutions.

\begin{theorem}\label{thm:mf_implicit_bias}
Given any data $Z \in \RR^{n \times d}$ and corresponding label $Y \in \RR^n$, 
suppose the loss $L(w)$ is in the form of $L(w) = \widetilde L(Zw)$ for 
some differentiable $\widetilde L: \RR^n \to \RR$.
Assume that initialized at $w(0) = \winit$, the mirror flow~\eqref{eq:MF} converges and the convergence point 
$w_\infty = \lim_{t \to \infty} w(t)$ satisfies $Zw_\infty = Y$,
then 
\begin{align*}
    D_R(w_\infty, w_0) = \min_{w: Zw = Y} D_R(w, w_0).
\end{align*}
\end{theorem}
See \Cref{sec:apdx_pre} for a proof.
The above theorem is the building block for proving the implicit bias induced by 
any commuting parametrization in overparametrized linear models~(see \Cref{thm:bias_linear_model}).

\section{Any gradient flow with commuting parametrization is a mirror flow}\label{sec:CGF_to_MF}

\subsection{Commuting parametrization}
We now formalize the notion of commuting parametrization.
We remark that $M$ is a smooth submanifold of $\RR^D$, and it is the domain of 
the parametrization $G$.

\begin{definition}[Commuting parametrization]\label{def:commuting_param}
Let $M$ be a smooth submanifold of $\RR^D$.
A $\mathcal{C}^2$ parametrization $G: M \to \RR^d$ is \emph{commuting} in a subset $S\subseteq M$ 
if and only if for any $i, j \in [d]$, the Lie bracket 
$[\nabla G_i, \nabla G_j](x) = 0$ for all $x \in S$. 
Moreover, we say $G$ is a \emph{commuting parametrization} if it is commuting in the 
entire $M$.
\end{definition}

In particular, when  $M$ is an open subset of $\RR^d$, $\{\nabla G_i\}_{i=1}^d$ 
are ordinary gradients in $\RR^D$, and the Lie bracket between 
any pair of $\nabla G_i$ and $\nabla G_j$ is given by 
\begin{align*} 
[\nabla G_i, \nabla G_j](x) = \nabla^2 G_j(x) \nabla G_i(x) - \nabla^2 G_i(x) \nabla G_j(x).
\end{align*}
This provides an easy way to check whether $G$ is commuting or not.

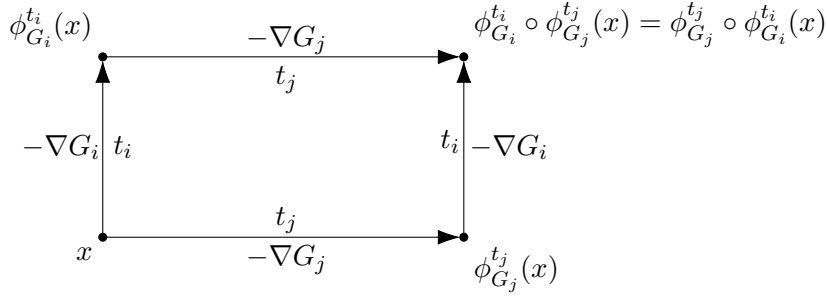
\begin{figure}[t]
\vspace{-1cm}
    \centering
    \tikzset{mycirc/.style={circle,fill=black,inner sep=0pt,minimum size=3pt}}
    \begin{tikzpicture}[scale=1.5]
    \draw[-{Latex[length=3mm,width=2mm]}](0, 0) -- (0, 1.57);
    \node (A) at (-0.4, 0.8) {{ $-\nabla G_i$}};
    \node (A) at (0.15, 0.8) {{ $t_i$}};
    \draw[-{Latex[length=3mm,width=2mm]}](0, 0) -- (3.17, 0);
    \node (A) at (1.6, -0.2) {{ $-\nabla G_j$}};
    \node (A) at (1.6, 0.15) {{ $t_j$}};
    \draw[-{Latex[length=3mm,width=2mm]}] (0, 1.6) -- (3.17, 1.6);
    \node (A) at (1.6, 1.75) {{ $-\nabla G_j$}};
    \node (A) at (1.6, 1.4) {{ $t_j$}};
    \draw[-{Latex[length=3mm,width=2mm]}] (3.2, 0) -- (3.2, 1.57);
    \node (A) at (3.6, 0.8) {{$-\nabla G_i$}};
    \node (A) at (3.05, 0.85) {{ $t_i$}};
    \filldraw[black] (0, 0) circle (1pt) node[anchor=north east]{$x$};
    \filldraw[black] (0, 1.6) circle (1pt) node[anchor=south east]{$\phi_{G_i}^{t_i}(x)$};
    \filldraw[black] (3.2, 0) circle (1pt) node[anchor=north west]{$\phi_{G_j}^{t_j}(x)$};
    \filldraw[black] (3.2, 1.6) circle (1pt) node[anchor=south west]{$\phi_{G_i}^{t_i}
        \circ \phi_{G_j}^{t_j}(x) = \phi_{G_j}^{t_j} \circ \phi_{G_i}^{t_i}(x)$};
    \end{tikzpicture}
    \caption{Illustration of commuting parametrizations.
    Suppose $G: M \to \RR^d$ is a commuting parametrization satisfying \Cref{assump:domain}, then starting from any $x \in M$, first moving along $-\nabla G_i$ for time $t_i$ then moving along $-\nabla G_j$ for time $t_j$ yields the same result as first moving along $-\nabla G_j$ for time $t_j$ then moving along $-\nabla G_i$ for time $t_i$ does, i.e., $\phi_{G_i}^{t_i} \circ \phi_{G_j}^{t_j}(x) = \phi_{G_j}^{t_j} \circ \phi_{G_i}^{t_i}(x)$.}
    \label{fig:commuting_vec_fields}
\end{figure}

The above definition of commuting parametrizations builds upon the differential 
properties of the gradient vector fields $\{\nabla G_i\}_{i=1}^d$, where 
each Lie bracket $[\nabla G_i, \nabla G_j]$ characterizes the change of 
$\nabla G_j$ along the flow generated by $\nabla G_i$.
In particular, when $G$ is a commuting parametrization satisfying \Cref{assump:domain}, 
it is further equivalent 
to a characterization of `commuting' in the integral form, as summarized in
\Cref{thm:commuting_equivalence}.
Also see \Cref{fig:commuting_vec_fields} for an illustration.

\begin{theorem}[Adapted from Theorem 9.44 in~\citet{lee2013smooth}]\label{thm:commuting_equivalence}
Let $M$ be a smooth submanifold of $\RR^D$ and $G: M \to \RR^d$ be a $\cC^2$
parametrization.
For any $i, j \in [d]$, $[\nabla G_i, \nabla G_j](x) = 0$ for all $x \in M$ if and only if for any $x \in M$,  whenever both $\phi_{G_i}^s \circ \phi_{G_j}^t(x)$ and $\phi_{G_j}^t \circ \phi_{G_i}^s(x)$ are well-defined for all $(s,t)$ in some rectangle $\cI_1 \times \cI_2$ where $\cI_1, \cI_2 \subseteq \RR$ are open intervals, it holds that $\phi_{G_i}^s \circ \phi_{G_j}^t(x) = \phi_{G_j}^t \circ \phi_{G_i}^s(x)$ for all $(s,t)\in \cI_1 \times \cI_2$.
\end{theorem}

Under \Cref{assump:domain}, \Cref{lem:domain} implie s that the domain of $\phi_{G_i}^s \circ \phi_{G_j}^t(x)$ is exactly $\cI_i(x)\times \cI_j(x)$, and thus the above theorem simplifies into the following.

\begin{theorem}\label{thm:commuting_equivalence_with_assumption_domain}
Let $M$ be a smooth submanifold of $\RR^D$ and $G: M \to \RR^d$ be a $\cC^2$
parametrization satisfying \Cref{assump:domain}.
For any $i, j \in [d]$, $[\nabla G_i, \nabla G_j](x) = 0$ for all $x \in M$ if and only if for any $x \in M$,  it holds that $\phi_{G_i}^s \circ \phi_{G_j}^t(x) = \phi_{G_j}^t \circ \phi_{G_i}^s(x)$ for all $(s,t)\in \cI_1(x) \times \cI_2(x)$.
\end{theorem}
The commuting condition clearly holds when each $G_i$ only depends on a 
different subset of coordinates of $x$, because we then have 
$\nabla ^2G_i (\cdot)\nabla G_j(\cdot)\equiv 0$ for any distinct $i, j \in[d]$ as 
$\nabla^2 G_i$ and $\nabla G_j$ live in different subspaces of $\RR^D$.  
We call such $G$ \emph{separable parametrizations}, and this case covers all the 
previous examples~\citep{gunasekar2018implicit, vaskevicius2019implicit, 
amid2020winnowing, woodworth2020kernel,amid2020reparameterizing}.
Another interesting example is the \emph{quadratic parametrization}: 
We parametrize $w \in \RR^d$ by $G: \RR^D \to \RR^d$ where for each $i \in [d]$, 
there is a symmetric matrix $A_i \in \RR^{D \times D}$ such that 
$G_i(x) = \frac{1}{2} x^\top A_i x$.
Then each Lie bracket $[G_i, G_j](x) = (A_jA_i - A_iA_j)x$, and thus $G$ is a 
commuting parametrization if and only if matrices $\{A_i\}_{i=1}^d$ commute.

For concreteness, we analyze two examples below. 
The first one is both  a separable parametrization and a commuting quadratic 
parametrization. 
The second one is a quadratic parametrization but not commuting.
\begin{example}[$u^{\odot 2}-v^{\odot 2}$ parametrization,~\citealt{woodworth2020kernel}]\label{eg:commuting_param}
Parametrize $w \in \RR^d$ by $w = u^{\odot 2} - v^{\odot 2}$.
Here $D = 2d$, and the parametrization $G$ is given by $G(x) = u^{\odot 2} - v^{\odot 2}$
for $x = \binom{u}{v} \in \RR^D$.    
Since each $G_i(x)$ involves only $u_i$ and $v_i$, $G$ is a separable parametrization 
and hence a commuting parametrization. 
Meanwhile, each $G_i(x)$ is a quadratic form in $x$, and it can be directly 
verified that the matrices underlying these quadratic forms commute with each 
other.   
\end{example}

\newcommand{\xddots}{%
  \raise 4pt \hbox {.}
  \mkern 6mu
  \raise 1pt \hbox {.}
  \mkern 6mu
  \raise -2pt \hbox {.}
}

\begin{example}[Matrix factorization]\label{eg:matrix_factorization}
As a counter-example, consider two parametrizations for matrix factorization: 
$G(U) = UU^\top$ and $G(U,V)=UV^\top$, where $U,V\in\RR^{d\times r}$ and $d\ge 2, r\ge 1$.
These are both \emph{non-commuting} quadratic parametrizations. 
Here we only demonstrate for the parametrization $G(U) = UU^\top$, 
and $G(U,V)=UV^\top$ follows a similar argument.
For each $i, j \in [d]$, we define $E_{ij} \in \RR^d$ as the one-hot matrix with the $(i,j)$-th entry 
being $1$ and the rest being $0$, and denote $\overline E_{ij} = \frac{1}{2} (E_{ij} + E_{ji})$.  
For $r=1$, we have $G_{ij}(U) = U_iU_j = U^\top \overline E_{ij} U$ for any 
$i, j \in [d]$, so $G$ is a quadratic parametrization.
Note that $\overline E_{ii}\overline E_{ij} = \frac{1}{2}E_{ij}
\neq \frac{1}{2}E_{ji} = \overline E_{ij} \overline{E}_{ii}$ for all distinct 
$i,j \in [d]$, which implies that $[\nabla G_{ij}, \nabla G_{ii}] \neq 0$, so $G$ is 
non-commuting.
More generally, we can reshape $U$ as a vector 
$\overrightarrow{U}:= [U_{:1}^\top, \ldots, U_{:r}^\top]^\top \in \RR^{rd}$ 
where each $U_{:j}$ is the $j$-th column of $U$, and 
the resulting quadratic form for the $(i,j)$-entry of $G(U)$ corresponds to a 
block-diagonal matrix:
\begin{align*}
    G_{ij}(U) = (\overrightarrow{U})^\top  \begin{pmatrix}
        \overline{E}_{ij}\\
        &\xddots\\
        &&\overline{E}_{ij}
    \end{pmatrix}\overrightarrow{U}.
\end{align*} 
Therefore, $\nabla^2 G_{ij}$ does not commute with $\nabla^2 G_{ii}$ due to the 
same reason as in the rank-1 case. 
\end{example}

\begin{remark}
This non-commuting issue for general matrix factorization does not conflict with 
the theoretical analysis in \citet{gunasekar2018implicit} where the measurements are 
commuting, or equivalently, only involves diagonal elements, as $\{G_{ii}\}_{i=1}^d$ 
are indeed commuting parametrizations. 
\citet{gunasekar2018implicit} is the first to identify the above non-commuting 
issue and conjectured that the implicit bias result for diagonal measurements 
can be extended to the general case.
\end{remark}

\subsection{Main Equivalence Result}
Next, we proceed to present our analysis for gradient flow with commuting 
parametrization.
The following two lemmas highlight the special properties of commuting parametrizations.
\Cref{lem:gf_commuting} shows that the point reached by gradient flow with any 
commuting parametrization is determined by the integral of the negative gradient 
of the loss along the trajectory.

\begin{lemma}\label{lem:gf_commuting}
Let $M$ be a smooth submanifold of $\RR^D$ and $G: M \to \RR^d$ be a commuting 
 parametrization.
For any initialization $\xinit \in M$, consider the gradient flow for any 
time-dependent loss $L_\cdot \in \cL$ as in \Cref{def:time_loss}:
\begin{align*}
\diff x(t) =  -\nabla   (L_t \circ G)(x(t)) \diff t, \qquad x(0) = \xinit.
\end{align*}
Further define $\mu(t) = \int_0^t - \nabla L_t(G(x(s))) \diff s$.
Suppose $\mu(t) \in \cU(\xinit)$ for all $t \in [0, T)$ where 
$T \in \RR \cup \{\infty\}$, then it holds that $x(t) = \psi(\xinit; \mu(t))$ 
for all $t \in [0, T)$.
\end{lemma}

Based on \Cref{lem:gf_commuting}, the next key lemma reveals the
essential approach to find the Legendre function.

\begin{lemma}\label{lem:commuting_potential}
Let $M$ be a smooth submanifold of $\RR^D$ and $G: M \to \RR^d$ be a commuting 
and regular parametrization satisfying \Cref{assump:domain}.
Then for any $\xinit \in M$, there exists a Legendre function 
$Q: \RR^d \to \RR\cup\{\infty\}$ such that $\nabla Q(\mu) = G(\psi(\xinit;\mu))$ 
for all $\mu \in \cU(\xinit)$.
Moreover, let $R$ be the convex conjugate of $Q$, then $R$ is also a Legendre 
function and satisfies that $\inter(\dom R) = \Omega_w(\xinit;G)$ and
\begin{align*}
\nabla^2 R(G(\psi(\xinit; \mu))) 
= \big(\partial G(\psi(\xinit; \mu)) \partial G(\psi(\xinit;\mu))^\top\big)^{-1}
\end{align*}
for all $\mu \in \cU(\xinit)$.
\end{lemma}

Next, we present our main result on characterizing any gradient flow with commuting 
parametrization by a mirror flow.

\begin{theorem}\label{thm:commuting_to_mirror}
Let $M$ be a smooth submanifold of $\RR^D$ and $G: M \to \RR^d$ be a commuting 
and regular parametrization satisfying \Cref{assump:domain}.
For any initialization $\xinit \in M$, consider the gradient flow for any 
time-dependent loss function $L_t:\RR^d\to\RR$:
\begin{align*}
    \diff x(t) = - \nabla (L_t \circ G)(x(t)) \diff t, \qquad x(0)=\xinit.
\end{align*}
Define $w(t) = G(x(t))$ for all $ t\geq 0$, then the dynamics of $w(t)$ is a 
mirror flow with respect to the Legendre function $R$ given by 
\Cref{lem:commuting_potential}, i.e., 
\begin{align*}
    \diff \nabla R(w(t)) = - \nabla L_t(w(t))\diff t, \qquad w(0) = G(\xinit).
\end{align*}
Moreover, this $R$ only depends on the initialization $\xinit$ and 
the parametrization $G$, and is independent of the loss function $L_t$.
\end{theorem}

\begin{proof}[Proof of \Cref{thm:commuting_to_mirror}]
Recall that the gradient flow in the $x$-space governed by $-\nabla (L_t \circ G)(x)$ is
\begin{align*}
    \diff x(t) = - \nabla (L_t \circ G)(x(t)) \diff t
    = - \partial G(x(t))^\top \nabla L_t(G(x(t))) \diff t.
\end{align*}
Using $w(t) = G(x(t))$, the corresponding dynamics in the $w$-space 
is
\begin{align}\label{eq:w_dynamics}
\diff w(t) &= \partial G(x(t)) \diff x(t)
= - \partial G(x(t)) \partial G(x(t))^\top \nabla L_t(w(t)) \diff t.
\end{align}

By \Cref{lem:gf_commuting}, we know that the solution to the gradient flow 
satisfies $x(t) = \psi(\xinit; \mu(t))$ where $\mu(t) = \int_0^t - \nabla L_t(G(x(s)))
\diff s$.
Therefore, applying \Cref{lem:commuting_potential}, we get a Legendre function 
$R: \RR^d \to \RR \cup \{\infty\}$ with domain $\Omega_w(\xinit;G)$ such that
\begin{align*}
\nabla^2 R(w(t)) &= \nabla^2 R(G(\psi(\xinit;\mu(t)))) 
= \big(\partial G(\psi(\xinit;\mu(t))) \partial G(\psi(\xinit;\mu(t)))\big)^{-1}
\end{align*}
for all $t\geq 0$.
Then the dynamics of $w(t)$ in \eqref{eq:w_dynamics} can be rewritten as  
\begin{align*}
    \diff w(t) = - \nabla^2 R(w(t))^{-1} \nabla L_t(w(t)) \diff t,
\end{align*}
or equivalently,
\begin{align*}
    \diff \nabla R(w(t)) = - \nabla L_t(w(t)) \diff t,
\end{align*}
which is exactly the mirror flow with respect to $R$ initialized at 
$w(0) = G(\xinit)$.
Further note that the result of Lemma~\ref{lem:commuting_potential} is completely 
independent of the loss function $L_t$, and thus $R$ only depends on 
the initialization $\xinit$ and the parametrization $G$.
This finishes the proof.
\end{proof}

\Cref{thm:commuting_to_mirror} provides a sufficient condition for when a 
gradient flow with certain parametrization $G$ is simulating a mirror flow.
The next question is then: What are the necessary conditions on the parametrization 
$G$ so that it enables the gradient flow to simulate a mirror flow?
We provide a (partial) characterization of such $G$ in the following theorem.

\begin{theorem}[Necessary condition on smooth parametrization to be commuting]
\label{thm:necessary_condition} 
Let $M$ be a smooth submanifold of $\RR^D$ and $G: M \to \RR^d$ be a 
smooth parametrization. 
If for any $\xinit \in M$, there is a Legendre function $R$  such 
that for all time-dependent loss $L_t \in \cL$, the gradient flow under 
$L_t \circ G$ initialized at $\xinit$ can be written as the mirror flow 
under $L_t$ with respect to $R$, then $G$ must be a regular 
parametrization, and it also holds that for each $x\in M$,
\begin{align}\label{eq:necessary_condition}
	\mathrm{Lie}^{\ge 2}(\partial G)\big\vert_{x} \subseteq \mathrm{ker}(\partial G(x)),
\end{align}
where $\mathrm{Lie}^{\ge K}(\partial G):=
\mathrm{span}\big\{[[[[ \nabla G_{j_{1}}, \nabla G_{j_2}], \ldots],
\nabla G_{j_{k-1}}],\nabla G_{j_k}] \mid k\geq K, \forall i \in [k], j_i\in [d]\}$ 
is the subset of the Lie algebra generated by the gradients of coordinate functions 
of $G$ only containing elements of order higher than $K$, and $\mathrm{ker}(\partial G(x))$ is the orthogonal complement of 
$\mathrm{span}(\{\nabla G_i(x)\}_{i=1}^d)$ in $\RR^D$.
\end{theorem}

Note the necessary condition in \eqref{eq:necessary_condition} is weaker than 
assuming that G is a commuting parametrization, and we conjecture that it is indeed sufficient. 
\begin{conjecture}
The claim in \Cref{thm:commuting_to_mirror} still holds, if we relax the 
commuting assumption to that 
$\mathrm{Lie}^{\ge 2}(\partial G)\big\vert_{x} \subseteq \mathrm{ker}(\partial G(x))$ 
for all $x\in M$.
\end{conjecture}

With the above necessary condition  \eqref{eq:necessary_condition}, we can formally refute the possibility that 
one can use mirror flow to characterize the implicit bias of gradient flow for 
matrix factorization in general settings, as summarized in 
\Cref{clrl:matrix_factorization_non_commuting}. 
It is also worth mentioning that \citet{li2019towards} constructed a concrete counter 
example showing that the implicit bias for commuting measurements, 
that gradient flow finds the solution with minimal nuclear norm, does not hold 
for the general case, where gradient flow could prefer the solution with minimal 
rank instead.

\begin{corollary}[Gradient flow for matrix factorization cannot be written as mirror flow]
\label{clrl:matrix_factorization_non_commuting}
For any $d, r \in \NN$, let $M$ be an open set in $\RR^{d\times r}$ and 
$G: M \to \RR^{d\times d}$ be a smooth parametrization given by $G(U) = UU^\top$. 
Then there exists a initial point $\xinit\in M$ and a time-dependent loss $L_t$ 
such that the gradient flow under $L_t\circ G$ starting from $U_{\textrm{init}}$ 
cannot be written as a mirror flow with respect to any Legendre function $R$ under 
the loss $L_t$.
\end{corollary}

\begin{proof}[Proof of \Cref{clrl:matrix_factorization_non_commuting}]
It turns out that the necessary condition in \Cref{thm:necessary_condition} is 
already violated  by only considering the Lie algebra spanned by $\{\nabla G_{11},\nabla G_{12}\}$. 
We follow the notation in \Cref{eg:matrix_factorization} to define each
$E_{ij} \in \RR^d$ as the one-hot matrix with the $(i,j)$-th entry being $1$, 
and denote $\overline E_{ij} = \frac{1}{2} (E_{ij} + E_{ji})$ and 
$\Delta_{ij} =E_{ij} - E_{ji}$. 
Then $[\nabla G_{11}, \nabla G_{12}](U) = 4(\overline E_{11}\overline E_{12} - \overline E_{12}\overline E_{11}) U= \Delta_{12}U$ 
and $[\nabla G_{11}, [\nabla G_{11}, \nabla G_{12}]] (U) = ( \overline E_{11}\Delta_{12} - \Delta_{12}\overline E_{11})U = \overline{E}_{12}U$.  Further noting that $\inner{[\nabla G_{11}, [\nabla G_{11}, \nabla G_{12}]]}{ \nabla G_{12}} = 2 \norm{\overline{E}_{12}U}_F^2 = \frac{1}{2}\sum_{i=1}^r (U_{1i}^2 + U_{2i}^2)$ must be positive at some $U$ in every open set $M$, 
by \Cref{thm:necessary_condition}, we know such $U_{\textrm{init}}$ and $L_t$ exist. 
Moreover, $L_t$ will only depend on $G_{11}(U)$ and $G_{12}(U)$.
\end{proof}

The following corollary shows that gradient flow with non-commuting parametrization 
cannot be mirror flow, when the dimension of the reachable set matches with that of the $w$-space.

\begin{corollary}
\label{cor:non_commuting}
Let $M$ be a smooth submanifold of $\RR^D$ whose dimension is at least $d$. 
Let $G: M \to \RR^d$ be a regular parametrization such that for any 
$\xinit \in M$, the following holds:
\begin{enumerate}
    \item[(a)] $\Omega_x(\xinit; G)$ is a submanifold of dimension $d$.
    \item[(b)] There is a Legendre function $R$ such that for any time-dependent loss 
    $L_t \in \cL$, the gradient flow governed by $-\nabla  (L_t \circ G)(x)$ with 
    initialization $\xinit$ can be written as a mirror flow with respect to $R$.
\end{enumerate}
Then $G$ must be a commuting parametrization. 
\end{corollary}

\begin{proof}[Proof of \Cref{cor:non_commuting}]
By the condition (b) and \Cref{thm:necessary_condition}, we know that each 
Lie bracket $[\nabla G_i, \nabla G_j] \in \ker(\partial G)$.
By the condition (a), we know that each Lie bracket 
$[\nabla G_i, \nabla G_j] \in \mathrm{span}\{\nabla G_i\}_{i=1}^d$.
Combining these two facts, we conclude that each $[\nabla G_i, \nabla G_j] \equiv 0$, 
so $G$ is a commuting parametrization.	
\end{proof}

Next, we establish the convergence of $w(t) = G(x(t))$ when $x(t)$ is given by 
some gradient flow with the commuting parametrization $G$.
Here we require that the convex function $R$ given by \Cref{lem:commuting_potential} 
is a Bregman function~(see definition in \Cref{apdx:convex_analysis}).
The proofs of \Cref{thm:convergence_commuting_flow},
\Cref{cor:convergence_commuting_flow_Rd} and 
\Cref{cor:convergence_commuting_quadratic_parametrization} are in \Cref{apdx:CGF_to_MF}.

\begin{theorem}\label{thm:convergence_commuting_flow}
Under the setting of \Cref{thm:commuting_to_mirror}, 
further assume that the loss $L$ is quasi-convex, $\nabla L$ is locally 
Lipschitz and $\argmin\{L(w) \mid w \in \dom R\}$ is non-empty where $R:\RR^d 
\to \RR\cup\{\infty\}$ is the convex function given by \Cref{lem:commuting_potential}.
Suppose $R$ is a Bregman function, then as $t \to \infty$, $w(t)$ 
converges to some $w^*$ such that $\nabla L(w^*)^\top (w - w^*) \geq 0$ for all $w \in \dom R$.
Moreover, if the loss function $L$ is convex, then $w(t)$ converges to a minimizer in 
$\overline{\dom R}$.
\end{theorem}

\begin{corollary}\label{cor:convergence_commuting_flow_Rd}
Under the setting of \Cref{thm:convergence_commuting_flow}, if the reachable set 
in the $w$-space satisfies $\Omega_w(\xinit;G) = \RR^d$, then $R$ is a Bregman function and all the statements in \Cref{thm:convergence_commuting_flow} hold.
\end{corollary}

\begin{theorem}\label{cor:convergence_commuting_quadratic_parametrization}
Under the setting of \Cref{thm:convergence_commuting_flow}, consider the 
commuting quadratic parametrization $G: \RR^D \to \RR^d$ where each 
$G_i(x) = \frac{1}{2} x^\top A_ix$, for symmetric matrices 
$A_1, A_2, \ldots, A_d\in\mathbb{R}^{D\times D}$ that commute with each other, 
\emph{i.e.}, $A_iA_j-A_jA_i=0$ for all $i,j\in[d]$. 
For any $\xinit \in \RR^D$, if 
$\{\nabla G_i(\xinit)\}_{i=1}^d = \{ A_i \xinit\}_{i=1}^d$ are linearly independent, 
then the following holds:
\begin{enumerate}
\item[(a)] For all $\mu \in \RR^d$, $\psi(\xinit;\mu ) = \exp(\sum_{i=1}^d \mu_i A_i) \xinit$ where $\exp(\cdot)$ 
is the matrix exponential defined as $\exp(A):= \sum_{k=0}^\infty \frac{A^k}{k!}$.
    
\item[(b)] For each $j \in [d]$ and all $\mu \in \RR^d$, 
$G_j(\psi(\xinit;\mu)) = \frac{1}{2} \xinit^\top \exp(\sum_{i=1}^d 2\mu_i A_i) 
A_j \xinit$.
    
\item[(c)] $Q(\mu) = \frac{1}{4}\norm{\psi(\xinit;\mu)}_2^2 = \frac{1}{4} 
\big\|\exp(\sum_{i=1}^d \mu_i A_i) \xinit\big\|_2^2$ is a Legendre function with domain $\RR^d$.
    
\item[(d)] $R$ is a Bregman function with $\dom R = \overline{\range \nabla Q}$ where 
$\range \nabla Q$ is the range of $\nabla Q$, and 
thus all the statements in \Cref{thm:convergence_commuting_flow} hold.
\end{enumerate}
\end{theorem}

\subsection{Solving underdetermined linear regression with commuting parametrization}
Next, we specialize to underdetermined linear regression problems to 
showcase our framework.

\paragraph{Setting: underdetermined linear regression.}
Let $\{(z_i,y_i)\}_{i=1}^n \subset \RR^d \times \RR$ be a dataset of size $n$. 
Given any parametrization $G$, the output of the linear model on the $i$-th data is $z_i^\top G(x)$.
The goal is to solve the regression for the label vector $Y = (y_1, y_2, \ldots, y_n)^\top$.
For notational convenience, we define $Z = (z_1, z_2, \ldots, z_n) \in \RR^{d \times n}$. 

We can apply \Cref{thm:mf_implicit_bias} to obtain the implicit bias of gradient
flow with any commuting parametrization.
\begin{theorem}\label{thm:bias_linear_model}
Let $M$ be a smooth submanifold of $\RR^d$ and $G: M \to \RR^d$ be a commuting 
and regular parametrization satisfying \Cref{assump:domain}.
Suppose the loss function $L$ satisfies $L(w) = \widetilde L(Zw)$ for some 
differentiable $\widetilde L: \RR^n \to \RR$.
For any $\xinit \in M$, consider the gradient flow 
\begin{align*}
    \diff x(t) = -\nabla (L \circ G)(x(t)) \diff t, \qquad x(0) = \xinit.
\end{align*}
There exists a convex function $R$  
(given by \Cref{lem:commuting_potential}, depending only on the initialization 
$\xinit$ and the parametrization $G$), such that for any dataset 
$\{(z_i, y_i)\}_{i=1}^n \subset \RR^d \times \RR$, if $w(t) = G(x(t))$ converges 
as $t \to \infty$ and the convergence point $w_\infty = \lim_{t \to \infty} w(t)$ 
satisfies $Z w_\infty = Y$, then 
\begin{align*}
    R(w_\infty) = \min_{w: Zw=Y} R(w),
\end{align*}
that is, gradient flow implicitly minimizes the convex regularizer $R$ among all
interpolating solutions.
\end{theorem}

\begin{proof}[Proof of \Cref{thm:bias_linear_model}]
By \Cref{thm:commuting_to_mirror}, $w(t)$ obeys the following mirror flow:
\begin{align*}
    \diff \nabla R(w(t)) = -\nabla L(w(t)) \diff t,\qquad w(0) = G(\xinit).
\end{align*}
Applying \Cref{thm:mf_implicit_bias} yields 
\begin{align*}
    D_R(w_\infty, G(\xinit)) = \min_{w: Zw=Y} D_R(w, G(\xinit)).
\end{align*}
Therefore, for any $w \in \inter(\dom R)$ such that $Zw = Y$, we have 
\begin{align*}
    &R(w_\infty) - R(G(\xinit)) - \langle \nabla R(G(\xinit), w_\infty - G(\xinit)\rangle\\
    &\qquad \leq R(w) - R(G(\xinit)) - \langle \nabla R(G(\xinit), w - G(\xinit)\rangle
\end{align*}
which can be reorganized as 
\begin{align}\label{eq:convergence_point}
    R(w_\infty) \leq R(w) - \langle \nabla R(G(\xinit)), w - w_\infty\rangle.
\end{align}
Note that by \Cref{lem:commuting_potential}, we also have 
\begin{align}\label{eq:R_property}
    \nabla R(G(\xinit)) = \nabla R(G(\psi(\xinit; 0))) = \nabla R(\nabla Q(0)) = 0
\end{align} 
where the last equality follows from the property of convex conjugate.
Combining \eqref{eq:convergence_point} and \eqref{eq:R_property}, we get 
$R(w_\infty) \leq R(w)$ for all $w \in \inter(\dom R)$ such that $Zw = Y$.
By the continuity of $R$, this property can be further extended to the entire
$\dom R$, and for any $w \notin \dom R$, we have $R(w) = \infty$ by definition, 
so $R(w_\infty) \leq R(w)$ holds trivially.
This finishes the proof.
\end{proof}

Note that the identity parametrization $w = G(x) = x$ is a commuting parametrization. 
Therefore, if we run the ordinary gradient flow on $w$ itself and it converges 
to some interpolating solution, then the convergence point is closest to the 
initialization in Euclidean distance among all interpolating solutions.
This recovers the well-known implicit bias of gradient flow for underdetermined 
regression.

Furthermore, we can recover the results on the quadratically overparametrized 
linear model studied in a series of papers~\citep{gunasekar2018implicit,woodworth2020kernel,azulay2021implicit},
as summarized in the following \Cref{cor:overparam_linear_model}.
Note that their results assumed convergence in order to characterize the 
implicit bias, whereas our framework enables us to directly prove the 
convergence as in \Cref{cor:convergence_commuting_quadratic_parametrization}, 
where the convergence guarantee is also more general 
than existing convergence results for \Cref{eg:commuting_param} 
in~\citet{pesme2021implicit,li2022what}.

\begin{corollary}\label{cor:overparam_linear_model}
Consider the underdetermined linear regression problem with data $Z \in \RR^{d \times n}$
and $Y \in \RR^n$.
Let $\widetilde{L}: \RR^n \to \RR$ be a differentiable loss function such that 
$\widetilde L$ is quasi-convex, $\nabla \widetilde{L}$ is locally Lipschitz, 
and $Y \in \RR^n$ is its unique global minimizer. 
Consider solving $\min_{w}\widetilde{L}(Zw)$ 
by running gradient flow on $L(w) = \widetilde L(Zw)$ with the quadratic 
parametrization $w = G(x) = u^{\odot 2} - v^{\odot 2}$ where 
$x = \binom{u}{v} \in \RR^{2d}_+$, for any initialization $\xinit \in \RR^{2d}_+$:
\begin{align*}
    \diff x(t) = - \nabla (L \circ G)(x(t)) \diff t, 
    \qquad x(0) = \xinit.
\end{align*}
Then as $t \to \infty$, $w(t) = G(x(t))$ converges to some $w_\infty$ such that 
$Zw_\infty = Y$ and 
\begin{align*}
    R(w_\infty) = \min_{w: Zw = Y} R(w)
\end{align*} 
where $R$ is given by
\begin{align*}
R(w) = \frac{1}{4} \sum_{i=1}^d \Big(w_i \arcsinh \Big(\frac{w_i}{2u_{0,i} v_{0,i}}\Big) 
- \sqrt{w_i^2 + 4 u^2_{0,i} v^2_{0,i}} - w_i\ln \frac{u_{0,i}}{v_{0,i}}\Big).
\end{align*}
\end{corollary}

\section{Every mirror flow is a gradient flow with commuting parametrization}\label{sec:MF_to_CGF}
Consider any smooth Legendre function $R: \RR^d \to \RR\cup\{\infty\}$, and 
recall the corresponding mirror flow:
\begin{align*}
    \diff \nabla R(w(t)) &= - \nabla L(w(t)) \diff t.
\end{align*}
Note that $\inter(\dom R)$ is a convex open set of $\RR^d$, hence a smooth 
manifold (see Example 1.26 in~\citet{lee2013smooth}).
Then $\nabla^2 R$ is a continuous positive-definite metric on $\inter(\dom R)$.
As discussed previously, the above mirror flow can be further rewritten as the 
Riemannian gradient flow on the Riemannian manifold $(\inter(\dom R), \nabla^2 R)$, 
\emph{i.e.}, 
\begin{align*}
    \diff w(t) = -\nabla^2 R(w(t))^{-1} \nabla L(w(t)) \diff t.
\end{align*}
The goal is to find a parametrization $G: U \to \RR^d$, where $U$ is 
an open set of $\RR^D$ and initialization $\xinit \in U$, such that the dynamics 
of $w(t) = G(x(t))$ can be induced by the gradient flow on $x(t)$ governed 
by $-\nabla (L \circ G)(x)$. 
Formally, we have the following result:

\begin{theorem}\label{thm:mirror_to_commuting}
Let $R: \RR^d \to \RR \cup \{\infty\}$ be a smooth Legendre function. 
There exist a smooth submanifold of $\RR^D$ denoted by $M$, an open neighborhood 
$U$ of $M$ and a smooth and regular parametrization $G: U \to \RR^d$ such that for  mirror flow 
on any time-dependent loss function $L_t$ with any initialization $\winit \in 
\inter(\dom R)$
\begin{align}\label{eq:mirror_flow_init}
    \diff \nabla R(w(t)) = - \nabla L_t(w(t)) \diff t, \quad w(0)=\winit,
\end{align}
it holds that $w(t) = G(x(t))$ for all $t \geq 0$ where $x(t)$ is given by the 
gradient flow under the objective $L_t \circ G$ initialized at $\xinit$, i.e., 
\begin{align}\label{eq:gradient_flow_init}
    \diff x(t) = - \nabla (L_t \circ G)(x(t)) \diff t, \quad x(0)=\xinit.
\end{align}
Moreover, $G$ restricted on $M$, denoted by $G\vert_M$ is a commuting and regular parametrization 
and $\partial G = \partial G\vert_M$ on $M$, which implies $x(t)\in M$
for all $t\ge 0$. 
If $R$ is further a mirror map, then $\{\nabla G_i|_M\}_{i=1}^d$ are complete 
vector fields on $M$. 
\end{theorem}

To illustrate the idea, let us first suppose such a smooth and regular 
parametrization $G$ exists and is a bijection between the reachable set $\Omega_x(\xinit;G)\subset \RR^D$ and $\inter(\dom R)$, 
whose inverse is denoted by $F$. 
It turns out that we can show
\begin{align*}
\partial F(w)^\top \partial F(w) 
= (\partial G(F(w)) \partial G(F(w))^\top)^{-1}
= \nabla^2 R(w)
\end{align*} 
where the second equality follows from the relationship between $R$ and $G$ as discussed in the introduction on~\Cref{eq:gf_w}.
Note that this corresponds to expressing the metric tensor $\nabla^2 R$ using an 
explicit map $F$, which is further equivalent to embedding the Riemannian 
manifold $(\inter(\dom R), \nabla^2 R)$ into a Euclidean space $(\RR^D,\overline{g})$ in a way 
that preserves its metric.
This refers to a notion called isometric embedding in differential geometry.

\begin{definition}[Isometric embedding]\label{defi:isometric_embedding}
Let $(M, g)$ be a Riemannian submanifold of $\RR^d$.
An \emph{isometric embedding} from $(M, g)$ to  $(\RR^D,\overline{g})$ is an differentiable injective map
$F: M \to \RR^D$ that preserves the metric in the sense that for any two tangent vectors $v, w \in T_x(M)$ we have 
$g_x(v, w) = \overline{g}_x(\partial F(x) v, \partial F(x) w)$ where the standard euclidean metric tensor $\overline{g}$ is defined as $\overline{g}_x(u,v) = \inner{u}{v}$ for all $u,v\in\RR^d$.
\end{definition}

Nash's embedding theorem is a classic result in differential geometry that 
guarantees the existence of isometric embedding of any Riemannian manifold into 
a Euclidean space with a plain geometry. 
\begin{theorem}[Nash's embedding theorem, \citealt{nash1956imbedding,gunther1991isometric}]
\label{thm:isometric_embedding}
Any $d$-dimensional Riemannian manifold has an isometric embedding to $(\RR^D,\overline{g})$ 
for some $D\geq d$.
\end{theorem}

The other way to understand \Cref{thm:commuting_to_mirror} is that we can view $\nabla^2 R(w)^{-1} \nabla L(w)$ as the gradient of $L$ with respect to metric tensor $g_R$, where $g^R$ is the Hessian metric induced by strictly convex function $R$ in the sense $g^R_x(u,v):= u^\top\nabla^2 R(x)v$ for any $u,v\in\RR^d$. It is well-known that gradient flow is invariant under isometric embedding and thus we can use Nash's embedding theorem to write the gradient flow on riemmanian manifold $(\inter(\dom R), g^R)$ as that on $(\RR^D, \overline{g})$.

\subsection{Existence of non-separable commuting parametrization}\label{sec:open_question}
Despite the recent line of works on the connection between mirror descent and 
gradient descent~\citep{gunasekar2018characterizing,amid2020winnowing,amid2020reparameterizing,azulay2021implicit,ghai2022non}, 
so far we have not seen any concrete example of 
non-separable parametrizaiton (in the sense of \Cref{defi:separable_general}) 
such that the reparametrized gradient flow can be written as a mirror flow. 
In this subsection, we discuss how we can use \Cref{thm:mirror_to_commuting} to construct non-separable, yet commuting parametrizations.

\begin{definition}[Separable parametrization in the general sense]\label{defi:separable_general}
Let $M$ be an open subset of $\RR^D$. 
We say a function $G:M\to \RR^d$ is a \emph{generalized separable parametrization} 
if and only if there exist $d$ projection matrices $\{P_i\}_{i=1}^d$ 
satisfying $\sum_{i=1}^d P_i = I_d$, $P_iP_j = \ind\{i=j\} \cdot P_i$, a function 
$\widehat{G}:M\to \RR^d$ satisfying $\widehat G_i(x) = \widehat G_i(P_i x)$, 
a matrix $A\in \RR^{d\times d}$ and a vector $b\in\RR^{d}$, such that 
\begin{align*}
	G(x) = A \widehat{G}(x)+ b, \qquad \forall x \in M . 
\end{align*}
\end{definition}

Given the above definition, it is easy to check that $\widehat G$ is a commuting parametrization as 
$\nabla^2 \widehat G_i \nabla \widehat G_j = P_i \nabla^2 \widehat G_i P_i\cdot P_j \nabla \widehat G_j\equiv0$ for all $i\neq j$, so each Lie 
bracket $[\nabla G_i, \nabla G_j]$ is also $0$ by the linearity. 

As a concrete example, for matrix sensing with commutable measurement $A_1,\ldots, A_m \in \RR^{d\times d}$, 
let $V = (v_1, \ldots, v_d) \in \RR^{d \times d}$ be a common eigenvector matrix 
for $\{A_i\}_{i=1}^m$ such that we can write 
$A_i = V\Sigma_i V^\top = \sum_{j=1}^d \sigma_{i,j}v_iv_i^\top$ for each $i \in [m]$. 
With parametrization $G:\RR^{d\times r}\to d$ where each $G_i(U) = v_i^\top UU^\top v_i $, 
we can write $\langle A_i, UU^\top\rangle = \sum_{j=1}^d \sigma_{i,j} G_j(U)$.

However, the bad news is that separable commuting parametrizations can only express a restricted class of Legendre functions. 
It is easy to see $\partial \hat G (x) \partial \hat G(x)^\top$ must be diagonal for every $x$. 
Thus $\partial G(x) \partial G(x)^\top$ are simultaneously diagonalizable for all $x$, 
and so are the Hessian of the corresponding Legendre function (given by \Cref{lem:commuting_potential}). 
There are interesting Legendre functions that does not always have their Hessians simultaneously diagonalizable, 
such as  
\begin{align*}
R(w) = \sum_{i=1}^d w_i(\ln w_i-1)
+ \bigg(1-\sum_{i=1}^d w_i\bigg)\bigg(\ln \bigg(1-\sum_{i=1}^d w_i\bigg)-1\bigg),
\end{align*}  
where each $w_i> 0$ and  $\sum_{i=1}^d w_i< 1$.
We can check that  $\nabla R(w) = \sum_{i=1}^d \ln \frac{w_i}{1-\sum_{i=1}^d w_i}$ 
and $\nabla^2 R(w) = \diag(w^{\odot (-1)}) + \ind_d\ind_d^\top$. 
It is proposed as an open problem by \cite{amid2020reparameterizing} that whether 
we can find a parametrization $G$ such that the reparametrized gradient flow in the $x$-space
simulates the mirror flow in the $w$-space with respect to the aforementioned Legendre function $R$. 

Our \Cref{thm:mirror_to_commuting} answers the open problem by \cite{amid2020reparameterizing} 
affirmatively since it shows every mirror flow can be written as some reparametrized gradient flow. 
According to the previous discussion, every mirror flow for Lengendre function whose Hessian cannot be simultaneously diagonalized always induces a non-separable commuting parametrization. 
But this type of construction has two caveats: First, the construction of the 
Legendre function uses Nash's Embedding theorem, which is implicit and hard to implement; 
second, the parametrization given by \Cref{thm:mirror_to_commuting}, though defined on an open set in $\RR^D$, is only commuting on the reachable set, which is a $d$-dimensional submanifold of $\RR^D$. 
This is different from all the natural examples of commuting parametrizations 
which are commuting on an open set, leading to the following open question.

\vspace{0.1in}
\noindent \textbf{Open Question:} Is there any smooth, regular, commuting, yet non-separable (in the sense of \Cref{defi:separable_general}) parametrization from an open subset of $\RR^D$ to $\RR^d$, for some integers $D$ and $d$? 

\begin{theorem}\label{thm:open_question_proof_D_1}
All smooth, regular and commuting parametrizations are non-separable when $D=1$.	
\end{theorem}

\begin{proof}[Proof of \Cref{thm:open_question_proof_D_1}]
Note that $[\nabla G_i,\nabla G_j] \equiv 0$ implies that all $ G_i$ share the 
same set of stationary points, \emph{i.e.}, $\{x \in \RR \mid \nabla G_i(x) =0\}$ is the 
same for all $i \in [d]$. 
Since $D=1$, without loss of generality, we can assume  $G_i'(x) = \nabla G_i(x)>0$ for all $x\in M$ and $i\in [d]$ since $G$ is regular. 
Then it holds that $\sign(G'_i)(\ln |G'_i|)' = \sign(G'_j)(\ln |G'_j|)'$, which 
implies that $|G_i'|/|G_j'|$ is equal to some constant independent of $x$. 
This completes the proof.
\end{proof}

\begin{remark}
We note that the assumption that the parametrization is regular is necessary for 
the open question to be non-trivial. 
Otherwise, consider the following example with $D=1$ and $d=2$:
Let $f_1,f_2:\RR\to \RR$ be any smooth function supported on $(0,1)$ and $(1,2)$ respectively. 
Define $G_i(x) = \int_{0}^x f_i(t)\diff t$ for all $x\in\RR$. 
Then parametrization $G$ is not separable.
\end{remark}

\section{Conclusion}
We presented a framework that characterizes when gradient descent with proper paramterization becomes equivalent to mirror descent.
In the limit of infinitesimal step size, we identify a notion named commuting parametrization such that any gradient flow (i.e., the continuous analog of gradient descent) with a commuting parametrization is equivalent to a mirror flow (i.e., the continuous analog of mirror descent) in the original parameter space with respect to a Legendre function that depends only on the initialization and the parametrization.
Conversely, we use Nash's embedding theorem to show that any mirror flow can be characterized by a gradient flow in the reparametrized space with a commuting parametrization.
Using our framework, we recover and generalize results on the implicit bias of gradient descent in a series of existing works, including a rigorous and general proof of convergence. We also provide a  necessary condition for the parametrization such that gradient flow in the reparametrized space is equivalent to a mirror flow in the original. However, the necessary condition is slightly weaker than commuting parametrization and it is left for future work to close the gap.

\section*{Acknowledgement}
This work was supported by NSF, DARPA/SRC, Simons Foundation, and ONR.
ZL acknowledges support of Microsoft Research PhD Fellowship and JDL acknowledges support of the ARO under MURI Award W911NF-11-1-0304, 
the Sloan Research Fellowship, NSF CCF 2002272, NSF IIS 2107304, 
ONR Young Investigator Award, and NSF CAREER Award 2144994.

\bibliography{reference}

\begin{thebibliography}{65}
\expandafter\ifx\csname natexlab\endcsname\relax\def\natexlab#1{#1}\fi
\expandafter\ifx\csname url\endcsname\relax
  \def\url#1{\texttt{#1}}\fi
\expandafter\ifx\csname urlprefix\endcsname\relax\def\urlprefix{URL }\fi

\bibitem[{Allen-Zhu et~al.(2019{\natexlab{a}})Allen-Zhu, Li and
  Liang}]{allen2019learning}
\textsc{Allen-Zhu, Z.}, \textsc{Li, Y.} and \textsc{Liang, Y.}
  (2019{\natexlab{a}}).
\newblock Learning and generalization in overparameterized neural networks,
  going beyond two layers.
\newblock \textit{Advances in neural information processing systems} .

\bibitem[{Allen-Zhu et~al.(2019{\natexlab{b}})Allen-Zhu, Li and
  Song}]{allen2019convergence}
\textsc{Allen-Zhu, Z.}, \textsc{Li, Y.} and \textsc{Song, Z.}
  (2019{\natexlab{b}}).
\newblock A convergence theory for deep learning via over-parameterization.
\newblock In \textit{International Conference on Machine Learning}. PMLR.

\bibitem[{Alvarez et~al.(2004)Alvarez, Bolte and Brahic}]{alvarez2004hessian}
\textsc{Alvarez, F.}, \textsc{Bolte, J.} and \textsc{Brahic, O.} (2004).
\newblock Hessian riemannian gradient flows in convex programming.
\newblock \textit{SIAM journal on control and optimization} \textbf{43}
  477--501.

\bibitem[{Amid and Warmuth(2020{\natexlab{a}})}]{amid2020reparameterizing}
\textsc{Amid, E.} and \textsc{Warmuth, M.~K.} (2020{\natexlab{a}}).
\newblock Reparameterizing mirror descent as gradient descent.
\newblock \textit{Advances in Neural Information Processing Systems}
  \textbf{33} 8430--8439.

\bibitem[{Amid and Warmuth(2020{\natexlab{b}})}]{amid2020winnowing}
\textsc{Amid, E.} and \textsc{Warmuth, M.~K.} (2020{\natexlab{b}}).
\newblock Winnowing with gradient descent.
\newblock In \textit{Conference on Learning Theory}. PMLR.

\bibitem[{Arora et~al.(2019{\natexlab{a}})Arora, Cohen, Hu and
  Luo}]{arora2019implicit}
\textsc{Arora, S.}, \textsc{Cohen, N.}, \textsc{Hu, W.} and \textsc{Luo, Y.}
  (2019{\natexlab{a}}).
\newblock Implicit regularization in deep matrix factorization.
\newblock \textit{Advances in Neural Information Processing Systems}
  \textbf{32}.

\bibitem[{Arora et~al.(2019{\natexlab{b}})Arora, Du, Hu, Li and
  Wang}]{arora2019fine}
\textsc{Arora, S.}, \textsc{Du, S.}, \textsc{Hu, W.}, \textsc{Li, Z.} and
  \textsc{Wang, R.} (2019{\natexlab{b}}).
\newblock Fine-grained analysis of optimization and generalization for
  overparameterized two-layer neural networks.
\newblock In \textit{International Conference on Machine Learning}. PMLR.

\bibitem[{Azulay et~al.(2021)Azulay, Moroshko, Nacson, Woodworth, Srebro,
  Globerson and Soudry}]{azulay2021implicit}
\textsc{Azulay, S.}, \textsc{Moroshko, E.}, \textsc{Nacson, M.~S.},
  \textsc{Woodworth, B.~E.}, \textsc{Srebro, N.}, \textsc{Globerson, A.} and
  \textsc{Soudry, D.} (2021).
\newblock On the implicit bias of initialization shape: Beyond infinitesimal
  mirror descent.
\newblock In \textit{International Conference on Machine Learning}. PMLR.

\bibitem[{Bauschke et~al.(1997)Bauschke, Borwein et~al.}]{bauschke1997legendre}
\textsc{Bauschke, H.~H.}, \textsc{Borwein, J.~M.} \textsc{et~al.} (1997).
\newblock Legendre functions and the method of random bregman projections.
\newblock \textit{Journal of convex analysis} \textbf{4} 27--67.

\bibitem[{Beck and Teboulle(2003)}]{beck2003mirror}
\textsc{Beck, A.} and \textsc{Teboulle, M.} (2003).
\newblock Mirror descent and nonlinear projected subgradient methods for convex
  optimization.
\newblock \textit{Operations Research Letters} \textbf{31} 167--175.

\bibitem[{Blanc et~al.(2020)Blanc, Gupta, Valiant and
  Valiant}]{blanc2020implicit}
\textsc{Blanc, G.}, \textsc{Gupta, N.}, \textsc{Valiant, G.} and
  \textsc{Valiant, P.} (2020).
\newblock Implicit regularization for deep neural networks driven by an
  ornstein-uhlenbeck like process.
\newblock In \textit{Conference on learning theory}. PMLR.

\bibitem[{Bregman(1967)}]{bregman1967relaxation}
\textsc{Bregman, L.~M.} (1967).
\newblock The relaxation method of finding the common point of convex sets and
  its application to the solution of problems in convex programming.
\newblock \textit{USSR computational mathematics and mathematical physics}
  \textbf{7} 200--217.

\bibitem[{Bubeck et~al.(2015)}]{bubeck2015convex}
\textsc{Bubeck, S.} \textsc{et~al.} (2015).
\newblock Convex optimization: Algorithms and complexity.
\newblock \textit{Foundations and Trends{\textregistered} in Machine Learning}
  \textbf{8} 231--357.

\bibitem[{Censor and Lent(1981)}]{censor1981iterative}
\textsc{Censor, Y.} and \textsc{Lent, A.} (1981).
\newblock An iterative row-action method for interval convex programming.
\newblock \textit{Journal of Optimization theory and Applications} \textbf{34}
  321--353.

\bibitem[{Chizat and Bach(2020)}]{chizat2020implicit}
\textsc{Chizat, L.} and \textsc{Bach, F.} (2020).
\newblock Implicit bias of gradient descent for wide two-layer neural networks
  trained with the logistic loss.
\newblock In \textit{Conference on Learning Theory}. PMLR.

\bibitem[{Chizat et~al.(2018)Chizat, Oyallon and Bach}]{chizat2018lazy}
\textsc{Chizat, L.}, \textsc{Oyallon, E.} and \textsc{Bach, F.} (2018).
\newblock On lazy training in differentiable programming.
\newblock \textit{arXiv preprint arXiv:1812.07956} .

\bibitem[{Crouzeix(1977)}]{crouzeix1977relationship}
\textsc{Crouzeix, J.-P.} (1977).
\newblock A relationship between the second derivatives of a convex function
  and of its conjugate.
\newblock \textit{Mathematical Programming} \textbf{13} 364--365.

\bibitem[{Damian et~al.(2021)Damian, Ma and Lee}]{damian2021label}
\textsc{Damian, A.}, \textsc{Ma, T.} and \textsc{Lee, J.} (2021).
\newblock Label noise sgd provably prefers flat global minimizers.
\newblock \textit{arXiv preprint arXiv:2106.06530} .

\bibitem[{Du et~al.(2019)Du, Lee, Li, Wang and Zhai}]{du2019gradient}
\textsc{Du, S.}, \textsc{Lee, J.}, \textsc{Li, H.}, \textsc{Wang, L.} and
  \textsc{Zhai, X.} (2019).
\newblock Gradient descent finds global minima of deep neural networks.
\newblock In \textit{International Conference on Machine Learning}. PMLR.

\bibitem[{Du et~al.(2018)Du, Zhai, Poczos and Singh}]{du2018gradient}
\textsc{Du, S.~S.}, \textsc{Zhai, X.}, \textsc{Poczos, B.} and \textsc{Singh,
  A.} (2018).
\newblock Gradient descent provably optimizes over-parameterized neural
  networks.
\newblock \textit{arXiv preprint arXiv:1810.02054} .

\bibitem[{Foote(1984)}]{foote1984regularity}
\textsc{Foote, R.~L.} (1984).
\newblock Regularity of the distance function.
\newblock \textit{Proceedings of the American Mathematical Society} \textbf{92}
  153--155.

\bibitem[{Ge et~al.(2021)Ge, Ren, Wang and Zhou}]{ge2021understanding}
\textsc{Ge, R.}, \textsc{Ren, Y.}, \textsc{Wang, X.} and \textsc{Zhou, M.}
  (2021).
\newblock Understanding deflation process in over-parametrized tensor
  decomposition.
\newblock \textit{Advances in Neural Information Processing Systems}
  \textbf{34}.

\bibitem[{Ghai et~al.(2022)Ghai, Lu and Hazan}]{ghai2022non}
\textsc{Ghai, U.}, \textsc{Lu, Z.} and \textsc{Hazan, E.} (2022).
\newblock Non-convex online learning via algorithmic equivalence.
\newblock \textit{arXiv preprint arXiv:2205.15235} .

\bibitem[{Gunasekar et~al.(2018{\natexlab{a}})Gunasekar, Lee, Soudry and
  Srebro}]{gunasekar2018characterizing}
\textsc{Gunasekar, S.}, \textsc{Lee, J.}, \textsc{Soudry, D.} and
  \textsc{Srebro, N.} (2018{\natexlab{a}}).
\newblock Characterizing implicit bias in terms of optimization geometry.
\newblock In \textit{International Conference on Machine Learning}. PMLR.

\bibitem[{Gunasekar et~al.(2018{\natexlab{b}})Gunasekar, Lee, Soudry and
  Srebro}]{gunasekar2018implicitconv}
\textsc{Gunasekar, S.}, \textsc{Lee, J.~D.}, \textsc{Soudry, D.} and
  \textsc{Srebro, N.} (2018{\natexlab{b}}).
\newblock Implicit bias of gradient descent on linear convolutional networks.
\newblock \textit{Advances in Neural Information Processing Systems}
  \textbf{31}.

\bibitem[{Gunasekar et~al.(2018{\natexlab{c}})Gunasekar, Woodworth,
  Bhojanapalli, Neyshabur and Srebro}]{gunasekar2018implicit}
\textsc{Gunasekar, S.}, \textsc{Woodworth, B.}, \textsc{Bhojanapalli, S.},
  \textsc{Neyshabur, B.} and \textsc{Srebro, N.} (2018{\natexlab{c}}).
\newblock Implicit regularization in matrix factorization.
\newblock In \textit{2018 Information Theory and Applications Workshop (ITA)}.
  IEEE.

\bibitem[{Gunasekar et~al.(2021)Gunasekar, Woodworth and
  Srebro}]{gunasekar2021mirrorless}
\textsc{Gunasekar, S.}, \textsc{Woodworth, B.} and \textsc{Srebro, N.} (2021).
\newblock Mirrorless mirror descent: A natural derivation of mirror descent.
\newblock In \textit{International Conference on Artificial Intelligence and
  Statistics}. PMLR.

\bibitem[{Gunther(1991)}]{gunther1991isometric}
\textsc{Gunther, M.} (1991).
\newblock Isometric embeddings of riemannian manifolds, kyoto, 1990.
\newblock In \textit{Proc. Intern. Congr. Math.} Math. Soc. Japan.

\bibitem[{HaoChen et~al.(2020)HaoChen, Wei, Lee and Ma}]{haochen2020shape}
\textsc{HaoChen, J.~Z.}, \textsc{Wei, C.}, \textsc{Lee, J.~D.} and \textsc{Ma,
  T.} (2020).
\newblock Shape matters: Understanding the implicit bias of the noise
  covariance.
\newblock \textit{arXiv preprint arXiv:2006.08680} .

\bibitem[{Jacot et~al.(2018)Jacot, Gabriel and Hongler}]{jacot2018neural}
\textsc{Jacot, A.}, \textsc{Gabriel, F.} and \textsc{Hongler, C.} (2018).
\newblock Neural tangent kernel: Convergence and generalization in neural
  networks.
\newblock \textit{arXiv preprint arXiv:1806.07572} .

\bibitem[{Jacot et~al.(2021)Jacot, Ged, Gabriel, {\c{S}}im{\c{s}}ek and
  Hongler}]{jacot2021deep}
\textsc{Jacot, A.}, \textsc{Ged, F.}, \textsc{Gabriel, F.},
  \textsc{{\c{S}}im{\c{s}}ek, B.} and \textsc{Hongler, C.} (2021).
\newblock Deep linear networks dynamics: Low-rank biases induced by
  initialization scale and l2 regularization.
\newblock \textit{arXiv preprint arXiv:2106.15933} .

\bibitem[{Ji et~al.(2021)Ji, Srebro and Telgarsky}]{ji2021fast}
\textsc{Ji, Z.}, \textsc{Srebro, N.} and \textsc{Telgarsky, M.} (2021).
\newblock Fast margin maximization via dual acceleration.
\newblock In \textit{International Conference on Machine Learning}. PMLR.

\bibitem[{Ji and Telgarsky(2018)}]{ji2018risk}
\textsc{Ji, Z.} and \textsc{Telgarsky, M.} (2018).
\newblock Risk and parameter convergence of logistic regression.
\newblock \textit{arXiv preprint arXiv:1803.07300} .

\bibitem[{Ji and Telgarsky(2019)}]{ji2019implicit}
\textsc{Ji, Z.} and \textsc{Telgarsky, M.} (2019).
\newblock The implicit bias of gradient descent on nonseparable data.
\newblock In \textit{Conference on Learning Theory}. PMLR.

\bibitem[{Ji and Telgarsky(2020)}]{ji2020directional}
\textsc{Ji, Z.} and \textsc{Telgarsky, M.} (2020).
\newblock Directional convergence and alignment in deep learning.
\newblock In \textit{Advances in Neural Information Processing Systems}
  (H.~Larochelle, M.~Ranzato, R.~Hadsell, M.~F. Balcan and H.~Lin, eds.),
  vol.~33. Curran Associates, Inc.

\bibitem[{Ji and Telgarsky(2021)}]{ji2021characterizing}
\textsc{Ji, Z.} and \textsc{Telgarsky, M.} (2021).
\newblock Characterizing the implicit bias via a primal-dual analysis.
\newblock In \textit{Algorithmic Learning Theory}. PMLR.

\bibitem[{Lang(2006)}]{lang2006introduction}
\textsc{Lang, S.} (2006).
\newblock \textit{Introduction to differentiable manifolds}.
\newblock Springer Science \& Business Media.

\bibitem[{Lee(2013)}]{lee2013smooth}
\textsc{Lee, J.~M.} (2013).
\newblock \textit{Introduction to Smooth Manifolds}.
\newblock Springer.

\bibitem[{Li et~al.(2018)Li, Ma and Zhang}]{li2018algorithmic}
\textsc{Li, Y.}, \textsc{Ma, T.} and \textsc{Zhang, H.} (2018).
\newblock Algorithmic regularization in over-parameterized matrix sensing and
  neural networks with quadratic activations.
\newblock In \textit{Conference On Learning Theory}. PMLR.

\bibitem[{Li et~al.(2019)Li, Wei and Ma}]{li2019towards}
\textsc{Li, Y.}, \textsc{Wei, C.} and \textsc{Ma, T.} (2019).
\newblock Towards explaining the regularization effect of initial large
  learning rate in training neural networks.
\newblock \textit{arXiv preprint arXiv:1907.04595} .

\bibitem[{Li et~al.(2020)Li, Luo and Lyu}]{li2020towards}
\textsc{Li, Z.}, \textsc{Luo, Y.} and \textsc{Lyu, K.} (2020).
\newblock Towards resolving the implicit bias of gradient descent for matrix
  factorization: Greedy low-rank learning.
\newblock In \textit{International Conference on Learning Representations}.

\bibitem[{Li et~al.(2022)Li, Wang and Arora}]{li2022what}
\textsc{Li, Z.}, \textsc{Wang, T.} and \textsc{Arora, S.} (2022).
\newblock What happens after {SGD} reaches zero loss? --a mathematical
  framework.
\newblock In \textit{International Conference on Learning Representations}.

\bibitem[{Lyu and Li(2019)}]{lyu2019gradient}
\textsc{Lyu, K.} and \textsc{Li, J.} (2019).
\newblock Gradient descent maximizes the margin of homogeneous neural networks.
\newblock \textit{arXiv preprint arXiv:1906.05890} .

\bibitem[{Lyu et~al.(2021)Lyu, Li, Wang and Arora}]{lyu2021gradient}
\textsc{Lyu, K.}, \textsc{Li, Z.}, \textsc{Wang, R.} and \textsc{Arora, S.}
  (2021).
\newblock Gradient descent on two-layer nets: Margin maximization and
  simplicity bias.
\newblock \textit{Advances in Neural Information Processing Systems}
  \textbf{34}.

\bibitem[{Moroshko et~al.(2020)Moroshko, Woodworth, Gunasekar, Lee, Srebro and
  Soudry}]{moroshko2020implicit}
\textsc{Moroshko, E.}, \textsc{Woodworth, B.~E.}, \textsc{Gunasekar, S.},
  \textsc{Lee, J.~D.}, \textsc{Srebro, N.} and \textsc{Soudry, D.} (2020).
\newblock Implicit bias in deep linear classification: Initialization scale vs
  training accuracy.
\newblock \textit{Advances in neural information processing systems}
  \textbf{33} 22182--22193.

\bibitem[{Nacson et~al.(2019)Nacson, Lee, Gunasekar, Savarese, Srebro and
  Soudry}]{nacson2019convergence}
\textsc{Nacson, M.~S.}, \textsc{Lee, J.}, \textsc{Gunasekar, S.},
  \textsc{Savarese, P. H.~P.}, \textsc{Srebro, N.} and \textsc{Soudry, D.}
  (2019).
\newblock Convergence of gradient descent on separable data.
\newblock In \textit{The 22nd International Conference on Artificial
  Intelligence and Statistics}. PMLR.

\bibitem[{Nash(1956)}]{nash1956imbedding}
\textsc{Nash, J.} (1956).
\newblock The imbedding problem for riemannian manifolds.
\newblock \textit{Annals of mathematics}  20--63.

\bibitem[{Nemirovskij and Yudin(1983)}]{nemirovskij1983problem}
\textsc{Nemirovskij, A.~S.} and \textsc{Yudin, D.~B.} (1983).
\newblock Problem complexity and method efficiency in optimization .

\bibitem[{Pesme et~al.(2021)Pesme, Pillaud-Vivien and
  Flammarion}]{pesme2021implicit}
\textsc{Pesme, S.}, \textsc{Pillaud-Vivien, L.} and \textsc{Flammarion, N.}
  (2021).
\newblock Implicit bias of sgd for diagonal linear networks: a provable benefit
  of stochasticity.
\newblock \textit{Advances in Neural Information Processing Systems}
  \textbf{34}.

\bibitem[{Qian and Qian(2019)}]{qian2019implicit}
\textsc{Qian, Q.} and \textsc{Qian, X.} (2019).
\newblock The implicit bias of adagrad on separable data.
\newblock \textit{Advances in Neural Information Processing Systems}
  \textbf{32}.

\bibitem[{Razin and Cohen(2020)}]{razin2020implicit}
\textsc{Razin, N.} and \textsc{Cohen, N.} (2020).
\newblock Implicit regularization in deep learning may not be explainable by
  norms.
\newblock \textit{Advances in neural information processing systems}
  \textbf{33} 21174--21187.

\bibitem[{Razin et~al.(2022)Razin, Maman and Cohen}]{razin2022implicit}
\textsc{Razin, N.}, \textsc{Maman, A.} and \textsc{Cohen, N.} (2022).
\newblock Implicit regularization in hierarchical tensor factorization and deep
  convolutional neural networks.
\newblock \textit{arXiv preprint arXiv:2201.11729} .

\bibitem[{Rockafellar(2015)}]{rockafellar2015convex}
\textsc{Rockafellar, R.~T.} (2015).
\newblock Convex analysis.
\newblock In \textit{Convex analysis}. Princeton university press.

\bibitem[{Soudry et~al.(2018)Soudry, Hoffer, Nacson, Gunasekar and
  Srebro}]{soudry2018implicit}
\textsc{Soudry, D.}, \textsc{Hoffer, E.}, \textsc{Nacson, M.~S.},
  \textsc{Gunasekar, S.} and \textsc{Srebro, N.} (2018).
\newblock The implicit bias of gradient descent on separable data.
\newblock \textit{The Journal of Machine Learning Research} \textbf{19}
  2822--2878.

\bibitem[{St{\"o}ger and Soltanolkotabi(2021)}]{stoger2021small}
\textsc{St{\"o}ger, D.} and \textsc{Soltanolkotabi, M.} (2021).
\newblock Small random initialization is akin to spectral learning:
  Optimization and generalization guarantees for overparameterized low-rank
  matrix reconstruction.
\newblock \textit{Advances in Neural Information Processing Systems}
  \textbf{34}.

\bibitem[{Sussmann(1973)}]{sussmann1973orbits}
\textsc{Sussmann, H.~J.} (1973).
\newblock Orbits of families of vector fields and integrability of
  distributions.
\newblock \textit{Transactions of the American Mathematical Society}
  \textbf{180} 171--188.

\bibitem[{Vaskevicius et~al.(2019)Vaskevicius, Kanade and
  Rebeschini}]{vaskevicius2019implicit}
\textsc{Vaskevicius, T.}, \textsc{Kanade, V.} and \textsc{Rebeschini, P.}
  (2019).
\newblock Implicit regularization for optimal sparse recovery.
\newblock \textit{Advances in Neural Information Processing Systems}
  \textbf{32} 2972--2983.

\bibitem[{Wang et~al.(2021{\natexlab{a}})Wang, Meng, Chen and
  Liu}]{wang2021implicit}
\textsc{Wang, B.}, \textsc{Meng, Q.}, \textsc{Chen, W.} and \textsc{Liu, T.-Y.}
  (2021{\natexlab{a}}).
\newblock The implicit bias for adaptive optimization algorithms on homogeneous
  neural networks.
\newblock In \textit{International Conference on Machine Learning}. PMLR.

\bibitem[{Wang et~al.(2021{\natexlab{b}})Wang, Meng, Zhang, Sun, Chen and
  Ma}]{wang2021momentum}
\textsc{Wang, B.}, \textsc{Meng, Q.}, \textsc{Zhang, H.}, \textsc{Sun, R.},
  \textsc{Chen, W.} and \textsc{Ma, Z.-M.} (2021{\natexlab{b}}).
\newblock Momentum doesn't change the implicit bias.
\newblock \textit{arXiv preprint arXiv:2110.03891} .

\bibitem[{Woodworth et~al.(2020)Woodworth, Gunasekar, Lee, Moroshko, Savarese,
  Golan, Soudry and Srebro}]{woodworth2020kernel}
\textsc{Woodworth, B.}, \textsc{Gunasekar, S.}, \textsc{Lee, J.~D.},
  \textsc{Moroshko, E.}, \textsc{Savarese, P.}, \textsc{Golan, I.},
  \textsc{Soudry, D.} and \textsc{Srebro, N.} (2020).
\newblock Kernel and rich regimes in overparametrized models.
\newblock In \textit{Conference on Learning Theory}. PMLR.

\bibitem[{Yang(2019)}]{yang2019scaling}
\textsc{Yang, G.} (2019).
\newblock Scaling limits of wide neural networks with weight sharing: Gaussian
  process behavior, gradient independence, and neural tangent kernel
  derivation.
\newblock \textit{arXiv preprint arXiv:1902.04760} .

\bibitem[{Yang and Hu(2021)}]{yang2021tensor}
\textsc{Yang, G.} and \textsc{Hu, E.~J.} (2021).
\newblock Tensor programs iv: Feature learning in infinite-width neural
  networks.
\newblock In \textit{International Conference on Machine Learning}. PMLR.

\bibitem[{Yun et~al.(2020)Yun, Krishnan and Mobahi}]{yun2020unifying}
\textsc{Yun, C.}, \textsc{Krishnan, S.} and \textsc{Mobahi, H.} (2020).
\newblock A unifying view on implicit bias in training linear neural networks.
\newblock \textit{arXiv preprint arXiv:2010.02501} .

\bibitem[{Zou et~al.(2020)Zou, Cao, Zhou and Gu}]{zou2020gradient}
\textsc{Zou, D.}, \textsc{Cao, Y.}, \textsc{Zhou, D.} and \textsc{Gu, Q.}
  (2020).
\newblock Gradient descent optimizes over-parameterized deep relu networks.
\newblock \textit{Machine Learning} \textbf{109} 467--492.

\bibitem[{Zou et~al.(2021)Zou, Wu, Braverman, Gu, Foster and
  Kakade}]{zou2021benefits}
\textsc{Zou, D.}, \textsc{Wu, J.}, \textsc{Braverman, V.}, \textsc{Gu, Q.},
  \textsc{Foster, D.~P.} and \textsc{Kakade, S.} (2021).
\newblock The benefits of implicit regularization from sgd in least squares
  problems.
\newblock \textit{Advances in Neural Information Processing Systems}
  \textbf{34} 5456--5468.

\end{thebibliography}
\bibliographystyle{ims}

\clearpage
\appendix

\section{Related basics for convex analysis}\label{apdx:convex_analysis}

We first introduce some additional notations.
For any function $f$, we denote its range (or image) by $\range f$.
For any set $S$, we use $\overline S$ to denote its closure.
For any matrix $\Lambda \in \RR^{d\times D}$ and set $S \subseteq \RR^D$, we define $\Lambda S = \{\Lambda x \mid x \in S\} \subseteq \RR^d$.

Below we collect some related basic definitions and results in convex analysis.
We refer the reader to~\citet{rockafellar2015convex} and~\citet{bauschke1997legendre} as main reference sources.
In particular, Sections 2, 3 and 4 in~\citet{bauschke1997legendre} provide a clear summary of the related concepts.

Here we consider a convex function $f: \RR^d \to \RR \cup \{\infty\}$ whose domain
is $\dom f = \{w \in \RR^d \mid f(w) < \infty\}$.
\textbf{From now on, we assume by default} that $f$ is continuous on $\dom f$, the interior of its domain $\inter(\dom f)$ is non-empty, and $f$ is differentiable on $\inter(\dom f)$.

The notions of essential smoothness and essential strict convexity defined below describe certain nice properties of a convex function (see Section 26 in~\citet{rockafellar2015convex}).

\begin{definition}[Essential smoothness and essential strict convexity]
If for any sequence $\{w_n\}_{n=1}^\infty \subset \inter(\dom f)$ going to the 
boundary of $\dom f$ as $n \to \infty$, it holds that $\|\nabla f(w_n)\| \to \infty$, then we say $f$ is \emph{essentially smooth}.
If $f$ is strictly convex on every convex subset of $\inter(\dom f)$, then we say $f$ is \emph{essentially strictly convex}.
\end{definition} 

The concept of \emph{convex conjugate} is critical in our derivation.
Specifically, given a convex function $f: \RR^d \to \RR \cup \{\infty\}$, its convex conjugate $f^*$ is defined as 
\begin{align*}
    f^*(w) = \sup_{y \in \RR^d} \langle w, y\rangle - f(y).
\end{align*}
The following results characterize the relationship between a convex function and its conjugate.
\begin{theorem}[Theorem 26.3, \citealt{rockafellar2015convex}]\label{thm:essential_smooth_strictly_convex}
A convex function $f$ is essentially strictly convex if and only if its convex conjugate $f^*$ is essentially smooth.
\end{theorem}

\begin{proposition}[Proposition 2.5, \citealt{bauschke1997legendre}]\label{prop:strictly_convex}
If $f$ is essentially strictly convex, then $\range \partial f = \inter(\dom f^*) = \dom \nabla f^*$, where $\partial f$ is the subgradient of $f$.
\end{proposition}

\begin{lemma}[Corollary 2.6, \citealt{bauschke1997legendre}]\label{lem:strict_convex}
If $f$ is essentially strictly convex, then it holds for all $w \in \inter(\dom f)$ that
$\nabla f(w) \in \inter(\dom f^*)$ and $\nabla f^*(\nabla f(w)) = w$.
\end{lemma}

The class of Legendre functions defined in \Cref{def:legendre} contains convex 
functions that are both essentially smooth and essentially strictly convex.

\begin{theorem}[Theorem 26.5, \citealt{rockafellar2015convex}]\label{thm:legendre}
A convex function $f$ is a Legendre function if and only if its conjugate $f^*$ is.
In this case, the gradient mapping $\nabla f: \inter(\dom f) \to \inter(\dom f^*)$ satisfies $(\nabla f)^{-1} = \nabla f^*$.
\end{theorem}

Next, we introduce the notion of Bregman function~\citep{bregman1967relaxation, censor1981iterative}.
It has been shown in \citet{bauschke1997legendre} that the properties of Bregman functions are crucial to prove the trajectory convergence of Riemannian gradient flow where the metric tensor is given by the Hessian of some Bregman function $f$.

\begin{definition}[Bregman functions; Definition 4.1, \citealt{alvarez2004hessian}]\label{def:bregman_function}
A function $f$ is called a \emph{Bregman function} if it satisfies the following properties:
\begin{itemize}
\item[(a)] $\dom f$  is closed. $f$ is strictly convex and continuous on $\dom f$. $f$ is $\cC^1$ on $\inter(\dom f)$.

\item[(b)] For any $w \in \dom f$ and $\alpha \in \RR$, $\{y \in \dom f \mid 
D_R(w, y) \leq \alpha\}$ is bounded.

\item[(c)] For any $w \in \dom f$ and sequence $\{w_i\}_{i=1}^\infty \subset \inter(\dom f)$ such that $\lim_{i \to \infty} w_i = w$, it holds that
$\lim_{i\to \infty}D_R(w, w_i) \to 0$.
\end{itemize}
\end{definition}

The following theorem provides a special sufficient condition for $f$ to be a Bregman function.
\begin{theorem}[Theorem 4.7, \citealt{alvarez2004hessian}]\label{thm:bregman}
If $f$ is a Legendre function with $\dom f = \RR^d$, then $\dom f^* = \RR^d$ implies that $f$ is a Bregman function.
\end{theorem}


The following theorem from~\citet{alvarez2004hessian} provides a convenient tool for proving the convergence of a Riemannian gradient flow.
\begin{theorem}[Theorem 4.2, \citealt{alvarez2004hessian}]
\label{thm:convergence_rgf}
Suppose $f: \RR^d \to \RR\cup\{\infty\}$ is a Bregman function and also a Legendre function, 
and satisfies that $f$ is twice continuously differentiable on $\inter(\dom f)$ and $\nabla^2 f$ is locally Lipschitz.
Consider the following Riemannian gradient flow:
\begin{align*}
    \diff w(t) = - \nabla^2 f(w(t))^{-1} \nabla L(w(t)) \diff t, \qquad w(0) = \winit \in \inter(\dom f)
\end{align*}
where the loss $L: \RR^d \to \RR$ satisfies that $L$ is quasi-convex, $\nabla L$ is locally Lipschitz, 
and $\argmin\{L(w) \mid w \in \dom f\}$ is non-empty.
Then as $t \to \infty$, $w(t)$ converges to some $w^* \in \dom f$ such that 
$\langle \nabla L(w^*), w - w^*\rangle \geq 0$ for all $w \in \dom f$.
If the loss $L$ is further convex, then $w^*$ is a minimizer of $L$ on $\dom f$.
\end{theorem}

\section{Omitted proofs in \Cref{sec:preliminary}}\label{sec:apdx_pre}
Here we first present the proof for the result on the domain of the flow induced 
by $G$.

\begin{proof}[Proof of \Cref{lem:domain}]
Fix any $x \in M$. 
For each $i \in [d]$, let $\cI_i(x)$ be the domain of $\phi_{G_j}^t(x)$ in terms 
of $t$.
If $\nabla G_i$ is a complete vector field on $M$ as in \Cref{def:complete_vec_field}, 
then $\cI_i(x) = \RR^d$, otherwise $\phi_{G_j}^t(x)$ is defined for $t$ in an 
open interval containing 0 (see, e.g., Theorem 2.1 in \citealt{lang2006introduction}).
Then we claim that for any distinct $j_1, j_2, \ldots, j_k \in [d]$ where $k \in [d]$, 
the set of all $(\mu_{j_1}, \ldots, \mu_{j_k}) \in \RR^k$ such that 
$\phi_{G_{j_1}}^{\mu_{j_1}} \circ \cdots \circ \phi_{G_{j_k}}^{\mu_{j_k}}(x)$ is 
well-defined is a hyperrectangle given by  
$\cI_{j_1}(x) \times \cI_{j_2}(x) \times \cdots \times \cI_{j_k}(x)$.
Then the desired result can be obtained by letting 
$(j_1, j_2, \ldots, j_d) = (1, 2, \ldots, d)$.
We prove the claim by induction over $k \in [d]$.

The base case for $k=1$ has already been established above.
Next, assume the claim holds for $1, 2, \ldots, k-1$ where $k\geq 3$, and we 
proceed to show it for $k$.
By the claim for $k-2$, 
$\phi_{G_{j_3}}^{\mu_{j_3}} \circ \cdots \circ \phi_{G_{j_k}}^{\mu_{j_k}}(x)$ is 
well-defined for $(\mu_{j_3}, \ldots, \mu_{j_k}) \in \cI_{j_3}(x)
\times \cdots \times \cI_{j_k}(x)$. 
For any such $(\mu_{j_3}, \ldots, \mu_{j_k})$, 
$\phi_{G_{j_1}}^t \circ \phi_{G_{j_3}}^{\mu_3} \circ \cdots \circ 
\phi_{G_{j_k}}^{\mu_{j_k}}(x)$ is well-defined for $t$ in and only in the open 
interval $\cI_{j_1}(x)$ by applying the claim for $k-1$, and similarly 
$\phi_{G_{j_2}}^t \circ \phi_{G_{j_3}}^{\mu_3} \circ \cdots \circ 
\phi_{G_{j_k}}^{\mu_{j_k}}(x)$ is also well-defined for $t$ in and only in the 
open interval $\cI_{j_2}(x)$.
Note that for any $(s, t) \in \cI_{j_1}(x) \times \cI_{j_2}(x)$, 
\begin{align*}
\phi_{G_{j_1}}^s \circ \phi_{G_{j_2}}^{-t} \circ \phi_{G_{j_2}}^t 
\circ \phi_{G_{j_3}}^{\mu_{j_3}} \circ \cdots \circ \phi_{G_{j_k}}^{\mu_{j_k}}(x)
\end{align*}
is well-defined, so by \Cref{assump:domain}, we see that 
\begin{align*}
\phi_{G_{j_2}}^{-t} \circ \phi_{G_{j_1}}^{s} \circ \phi_{G_{j_2}}^t 
\circ \phi_{G_{j_3}}^{\mu_{j_3}} \circ \cdots \circ \phi_{G_{j_k}}^{\mu_{j_k}}(x)
\end{align*}
is also well-defined, which further implies that 
$\phi_{G_{j_1}}^s \circ \phi_{G_{j_2}}^t \circ \phi_{G_{j_3}}^{\mu_{j_3}} 
\circ \cdots \circ \phi_{G_{j_k}}^{\mu_{j_k}}(x)$ is well-defined.
Therefore, we conclude that $\phi_{G_{j_1}}^{\mu_{j_1}} \circ \cdots 
\circ \phi_{G_{j_k}}^{\mu_{j_k}}(x)$ is well-defined for and only for 
$(\mu_{j_1}, \ldots, \mu_{j_k}) \in \cI_{j_1}(x) \times \cdots \times \cI_{j_k}(x)$.
This completes the induction and hence finishes the proof.
\end{proof}

Next, we provide the proof for the implicit bias of mirror flow summarized in \Cref{thm:mf_implicit_bias}.
We need the following lemma that characterizes the KKT conditions for minimizing a convex function $R$ in a linear subspace.

\begin{lemma}\label{lem:convex}
    For any convex function $R: \RR^d \to \RR \cup \{\infty\}$ and $Z \in \RR^{n \times d}$, 
    suppose $\nabla R(w^*) = Z^\top \lambda$ for some $\lambda \in \RR^n$,
    then
    \begin{align*}
        R(w^*) = \min_{w: Z(w - w^*) = 0} R(w).
    \end{align*}
\end{lemma}
\begin{proof}[Proof of \Cref{lem:convex}]
    Consider another convex function defined as $\widetilde R(w) = R(w) - w^\top 
    Z^\top \lambda$, then $\nabla \widetilde R(w^*) = \nabla R(w^*) - Z^\top \lambda
    = 0$, which implies that 
    \begin{align*}
        \tilde R(w^*) &= \min_{w \in \RR^d} R(w) - w^\top Z^\top \lambda\\
        &\leq \min_{w: Z(w - w^*) = 0} R(w) - w^\top Z^\top \lambda\\
        &= \min_{w: Z(w - w^*) = 0} R(w) - w^{* \top} Z^\top \lambda.
    \end{align*}
    Since $\widetilde R(w^*) = R(w^*) - w^{* \top} Z^\top \lambda$, it follows that 
    \begin{align*}
        R(w^*) \leq \min_{w: Z(w - w^*) = 0} R(w),
    \end{align*}
    and the equality is achieved at $w = w^*$.
    This finishes the proof.
\end{proof}

We then can prove \Cref{thm:mf_implicit_bias} by using \Cref{lem:convex}.
\begin{proof}[Proof of \Cref{thm:mf_implicit_bias}]
Since $L(w) = \widetilde L(Zw-Y)$, the mirror flow~\eqref{eq:MF} can be further written as
\begin{align*}
    \diff \nabla R(w(t)) = - Z^\top \nabla\widetilde L(Zw(t) - Y) \diff t.
\end{align*}
Integrating the above yields that for any $t\geq 0$,
\begin{align*}
    \nabla R(w(t)) - \nabla R(w_0) = - Z^\top \int_0^t
    \nabla \widetilde L(Z w(s) - Y) \diff s \in \text{span}(X^\top),
\end{align*}
which further implies that $\nabla R(w_\infty) - \nabla R(w_0) \in \text{span}(Z^\top)$.
Therefore, 
\begin{align*}
    \nabla D_R(w, w_0) |_{w = w_\infty} = \nabla R(w_\infty) - \nabla R(w_0)
\in \text{span}(Z^\top).
\end{align*}
Then applying \Cref{lem:convex} yields
\begin{align*}
    D_R(w_\infty, w_0) = \min_{w: Z(w-w_\infty)=0} D_R(w, w_0).
\end{align*}
This finishes the proof.
\end{proof}

\section{Omitted proofs in \Cref{sec:CGF_to_MF}}\label{apdx:CGF_to_MF}
Here we provide the omitted proofs in \Cref{sec:CGF_to_MF}, including four main parts: 
\begin{enumerate}
    \item[(1)] Properties of commuting parametrizations (\Cref{apdx:proof_commuting_param});
    \item[(2)] Necessary condition for a smooth parametrization to be commuting (\Cref{apdx:necessary_condition});
    \item[(3)] Convergence for  gradient flow with commuting parametrization (\Cref{apdx:convergence});
    \item[(4)] Results for the underdetermined linear regression (\Cref{apdx:linear_regression}).
\end{enumerate}

\subsection{Properties of commuting parametrizations}\label{apdx:proof_commuting_param}

We first show the representation formula for gradient flow with commuting parametrization
 given in \Cref{lem:gf_commuting}.
\begin{proof}[Proof of \Cref{lem:gf_commuting}]
Let $\mu(t)$ be given by the following differential equation:
\begin{align*}
    \diff \mu(t) = - \nabla L_t(G(\psi(\xinit; \mu(t)))) \diff t, \qquad \mu(0) = 0.
\end{align*}
For any $\mu \in \cU(x)$ and $j \in [d]$, $\mu + \delta e_j \in \cU(x)$ 
for all sufficiently small $\delta$, thus
\begin{align*}
    \frac{\partial}{\partial \mu_j} \psi(\xinit; \mu) &= \lim_{\delta \to 0} \frac{\psi(\xinit; \mu + \delta e_j) - \psi(\xinit; \mu)}{\delta}\\
    &= \lim_{\delta \to 0} \frac{\phi_{G_j}^\delta(\psi(\xinit;\mu)) - \psi(\xinit;\mu)}{\delta}\\
    &= \nabla G_j(\psi(\xinit;\mu))
\end{align*} 
where the second equality follows from the assumption that $G$ is a commuting 
parametrization and \Cref{thm:commuting_equivalence}.
Then we have $\frac{\partial\psi(\xinit; \mu)}{\partial \mu} 
= \partial G(\psi(\xinit;\mu))^\top$ for all $\mu \in \cU(\xinit)$, and thus 
when $\mu(t) \in \cU(\xinit)$, 
\begin{align*}
    \diff \psi(\xinit; \mu(t)) &= \frac{\partial \psi(\xinit; \mu(t))}{\partial \mu(t)} 
    \diff \mu(t)\\
    &= -\partial G(\xinit; \mu(t)) \nabla L_t(G(\psi(\xinit;\mu(t)))) \diff t\\
    &= -\nabla (L_t \circ G)(\psi(\xinit; \mu(t))) \diff t.
\end{align*}
Then since $\psi(\xinit; \mu(0)) = \xinit$ and $\psi(\xinit; \mu(t))$ follows 
the same differential equation and has 
the same initialization as $x(t)$, we have $x(t) \equiv \psi(\xinit; \mu(t))$
for all $t \in [0, T)$.
Therefore,
\begin{align*}
    \mu(t) &= \mu(0) + \int_0^t - \nabla L_t(G(\psi(\xinit; \mu(s)))) \diff s    
    = \int_0^t - \nabla L_t(G(x(s))) \diff s
\end{align*} 
for all $t \in [0, T)$, which completes the proof.
\end{proof}

Next, to prove \Cref{lem:commuting_potential}, we need the following lemma which provides a sufficient condition for a vector 
function to be gradient of some other function.

\begin{lemma}\label{lem:grad_vec_field}
Let $\Psi: C \to \RR^d$ be a differentiable function where $C$ is a simply connected
open subset of $\RR^d$. 
If for all $w \in C$ and any $i, j \in [d]$, $\frac{\partial}{\partial w_j} \Psi_i(w) 
= \frac{\partial}{\partial w_i} \Psi_j(w)$, then there exists some function 
$Q: C \to \RR$ such that $\Psi = \nabla Q$.
\end{lemma}

\begin{proof}[Proof of \Cref{lem:grad_vec_field}]
This follows from a direct application of Corollary 16.27 in \citet{lee2013smooth}.
\end{proof}

Based on the above results, we proceed to prove \Cref{lem:commuting_potential}.
\begin{proof}[Proof of \Cref{lem:commuting_potential}]
By \Cref{lem:domain}, $\cU(\xinit)$ is hyperrectangle, and hence is convex.
Next, recall that by the proof of \Cref{lem:gf_commuting}, we have 
$\frac{\partial\psi(\xinit; \mu)}{\partial \mu} 
= \partial G(\psi(\xinit;\mu))^\top$ for all $\mu \in \cU(\xinit)$.
Denoting $\Psi(\mu) = G(\psi(\xinit; \mu))$, we further have
\begin{align*}
    \partial\Psi(\mu) = \frac{\partial G(\psi(\xinit; \mu))}{\partial \psi(\xinit;\mu)} 
    \frac{\partial\psi(\xinit; \mu)}{\partial \mu}
    = \partial G(\psi(\xinit;\mu)) \partial G(\psi(\xinit; \mu))^\top, \quad\forall 
    \mu \in \cU(x).
\end{align*}
Since $G$ is regular, 
$\partial G(\psi(\xinit; \mu))$ is of full-rank for all $\mu \in \cU(\xinit)$, 
so $\partial\Psi$ is symmetric and positive definite for all $\mu \in \cU(\xinit)$, 
which implies that $\Psi$ is the gradient of some strictly convex function 
$Q: \RR^d \to \RR\cup\{\infty\}$ by \Cref{lem:grad_vec_field}.
This $Q$ satisfies that $\nabla Q(\mu) = \Psi(\mu) = G(\psi(\xinit; \mu))$ for 
all $\mu \in \cU(\xinit)$.
Therefore, $Q$ is a strictly convex function with $\dom \nabla Q = \cU(\xinit)$
and $\range\nabla Q = \Omega_w(\xinit; G)$.

Next, we show that $Q$ is essentially smooth.
If $\cU(\xinit) = \RR^d$, then $\dom Q = \RR^d$ and the boundary of $\dom Q$ is empty, so it is trivial that $Q$ is essentially smooth.
Otherwise, it suffices to show that for any $\mu$ on the boundary of $\dom Q$ and any sequence $\{\mu_k\}_{k=1}^\infty \subset \cU(\xinit)$ such that $\lim_{k \to \infty} \mu_k = \mu_\infty$, we have $\lim_{k \to \infty} \|\nabla Q(\mu_k)\|_2 = \infty$.
Since each $\nabla Q(\mu_k) = G(\psi(\xinit; \mu_k))$, we only need to show that $\lim_{k \to \infty} \|G(\psi(\xinit; \mu_k))\|_2 = \infty$.
Suppose otherwise, then $\{G(\psi(\xinit;\mu_k)\}_{k=1}^\infty$ is bounded.
Note that by \Cref{lem:gf_commuting}, let $H_k(x) = \inner{\mu_k}{G(x)}$, 
and we have
\begin{align*}
\psi(\xinit; \mu_k) = \phi_{-H_k}^1(\xinit) 
= \xinit + \int_0^1 \nabla H_k(\phi_{-H_k}^s(\xinit)) \diff s.
\end{align*}
Therefore,
\begin{align}\label{eq:dist_psi}
\|\psi(\xinit; \mu_k) - \xinit\|_2 
\leq \int_0^1 \big\|\nabla H_k(\phi_{-H_k}^s(\xinit))\big\|_2 \diff s 
\leq \sqrt{\int_0^1 \big\|\nabla H_k(\phi_{-H_k}^s(\xinit))\big\|^2_2 \diff s}.
\end{align}
where the second inequality follows from Cauchy-Schwarz inequality.
Further note that
\begin{align}\label{eq:psi_potential}
H_k(\psi(\xinit; \mu_k)) - H_k(\xinit)
&=  \int_0^1 \frac{\diff }{\diff s} H_k(\phi_{-H_k}^s(\xinit) ) \diff s\notag\\
&= \int_0^1 \bigg\langle\nabla H_k(\phi_{-H_k}^s(\xinit)), 
\frac{\diff \phi_{-H_k}^s(\xinit) }{\diff s}\bigg\rangle \diff s \notag\\
&= \int_0^1 \|\nabla H_k (\phi_{-H_k}^s(\xinit) )\|_2^2 \diff s.
\end{align}
Then combining \eqref{eq:dist_psi} and \eqref{eq:psi_potential}, we get 
\begin{align*}
\|\psi(\xinit; \mu_k) - \xinit\|_2 &\leq \sqrt{\langle \mu_k, G(\psi(\xinit;\mu_k))
- G(\xinit)\rangle}\\
&\leq \sqrt{\|\mu_k\|_2 \cdot \|G(\psi(\xinit; \mu_k)) - G(\xinit)\|_2},
\end{align*}
which implies that $\{\psi(\xinit;\mu_k)\}_{k=1}^\infty$ is bounded.
Then there exists a convergent subsequence of $\{\psi(\xinit;\mu_k)\}_{k=1}^\infty$,
and without loss of generality we assume that $\psi(\xinit; \mu_k)$ itself 
converges to some $x_\infty \in M$ as $k \to \infty$. 
Note that $\psi(x_\infty; \mu)$ is well-defined for $\mu$ in a small open 
neighborhood of $0$, and since $\lim_{k \to \infty} \psi(\xinit; \mu_k) = 
x_\infty$, for sufficiently large $k$, $\psi(\psi(\xinit;\mu_k); \mu)$ is 
well-defined for $\mu$ in a small neighborhood of $0$ that does not depend on 
$k$. Thus there exists some $\mu\in \RR^d$ such that $\mu_k + \mu \notin \cU(\xinit)$ but $\psi(\psi(\xinit;\mu_k); \mu)$ is 
well-defined for sufficiently large $k$.  But by \Cref{lem:domain} and \Cref{thm:commuting_equivalence},  $\psi(\psi(\xinit;\mu_k); \mu) = \psi(\xinit;\mu_k+\mu)$ and thus  $\mu_k + \mu\in \cU(\xinit)$, which leads to a contradiction.
Hence, we conclude that $Q$ is essentially smooth.

Combining the above, it follows that $Q$ is a Legendre function.
Let $R: \RR^d \to \RR\cup\{\infty\}$ be the convex conjugate of $Q$.
Then by \Cref{thm:legendre}, $R$ is also a Legendre function.
Note that for any $\mu \in \cU(\xinit)$, by the result in~\citet{crouzeix1977relationship}, 
we have 
\begin{align*}
\nabla^2 R(G(\psi(\xinit; \mu))) = \nabla^2 R(\nabla Q(\mu)) 
= \nabla^2 Q(\mu)^{-1} 
= (\partial G(\psi(\xinit; \mu)) \partial G(\psi(\xinit;\mu))^\top)^{-1}.
\end{align*}
Therefore, $R$ and $Q$ are both Legendre functions, and by 
\Cref{prop:strictly_convex}, we further have 
$\range \nabla R = \inter(\dom Q) = \dom \nabla Q = \cU(x)$ and conversely 
$\dom \nabla R = \range \nabla Q = \Omega_w(\xinit; G)$.
This finishes the proof.
\end{proof}

\subsection{Necessary condition for a smooth parametrization to be commuting}
\label{apdx:necessary_condition}

\begin{proof}[Proof of \Cref{thm:necessary_condition}]
Fix any initialization $\xinit \in M$, and let the Legendre function $R$ be given such 
that for all time-dependent loss $L_t$, the gradient flow under $L_t \circ G$
initialized at $x$ can be written as the mirror flow under 
$L_t$ with respect to the Legendre function $R$.
We first introduce a few notations that will be useful for the proof.
For any $s \in \RR$, we define a time-shifting operator $\cT_s$ such that for 
any time-dependent loss $L_t(\cdot)$, $(\cT_s L)_t(\cdot) = L_{t-s}(\cdot)$.
We say a time-dependent loss $L_t$ is supported on finite time if 
$L_t = \sum_{i=1}^k \ind_{t\in [t_i,t_{i+1})} L^{(i)}$ for some $k \geq 1$
where $t_1 = 0$, $t_{k+1} = \infty$ and
$L^{(k)}\equiv 0$, and we denote $\len(L) = t_k$. 
We further define the concatenation of two time-dependent loss $L_t, L_t'$ 
supported on finite time as $L\concate L' = L + \cT_{\len(L)} L'$. 
We also use $\overline{L}$ to denote the time-reverse of the time-dependent loss 
$L$ which is supported on finite time, that is, $\overline{L}_t = L_{\len(L)-t}$
for all $t \geq 0$.
For any $j \in [d]$ and $\delta > 0$, we define the following loss 
function 
\begin{align}\label{eq:ell}
\ell_t^{j, \delta}(w) = \ind_{0 \leq t \leq \delta} \cdot 
\langle e_j, w\rangle
\end{align}
where $e_j$ is the $j$-th canonical base of $\RR^d$.

Now for any $k \geq 2$, let $\{j_i\}_{i=1}^k$ be any sequence where each 
$j_i\in [d]$.
Then we recursively define a sequence of time-dependent losses as follows: 
First define $L^{1, \delta} = - \ell^{j_1, \delta}$, 
then sequentially for each $i = 2, 3, \ldots, k$, we define
\begin{align}\label{eq:L}
L^{i, \delta} = L^{i-1, \sqrt{\delta}}
\concate \left(- \ell^{j_i, \sqrt{\delta}}\right)
\concate \left(-\overline L^{i-1, \sqrt{\delta}}\right)
\concate \ell^{j_i, \sqrt{\delta}}
\end{align}
where we write $\overline L^{i-1, \sqrt{\delta}} = \overline{L^{i-1, \sqrt{\delta}}}$
for convenience.
Denote $\iota_i(\delta) = \len(L^{i, \delta})$ for each $i \in [k]$.
Then $\iota_1(\delta) = \delta$ and $\iota_i(\delta)
= 2\sqrt{\delta} + 2 \iota_{i-1}(\sqrt{\delta})$ for $i = 2, 3, \ldots, k$,
which further implies 
\begin{align*}
\iota_i(\delta) = \sum_{m=1}^{i-1} 2^m \delta^{1/2^m} + 2^{i-1} \delta^{1/2^{i-1}} 
\text{ for all } i \in [k].
\end{align*}
Moreover, for each $i = 2, 3, \ldots, k$, the gradient of $L^{i, \delta}$ with 
respect to $w$ is given by 
\begin{align}\label{eq:L_k}
\nabla L_t^{i, \delta}(w) &= \begin{cases}
\nabla L_t^{i-1, \sqrt{\delta}}(w) & 0 \leq t \leq \iota_{i-1}(\sqrt{\delta}),\\
- e_{j_i} & \iota_{i-1}(\sqrt{\delta}) < t \leq \iota_{i-1}(\sqrt{\delta}) + \sqrt{\delta},\\
- \nabla \overline{L}_t^{i-1, \sqrt{\delta}}(w) & \iota_{i-1}(\sqrt{\delta}) + \sqrt{\delta} < t 
\leq 2\iota_{i-1}(\sqrt{\delta}) + \sqrt{\delta},\\
e_{j_i} & 2\iota_{i-1}(\sqrt{\delta}) + \sqrt{\delta}
< t \leq 2\iota_{i-1}(\sqrt{\delta}) + 2\sqrt{\delta},\\
0 & t > 2\iota_{i-1}(\sqrt{\delta}) + 2\sqrt{\delta}.
\end{cases}
\end{align}
This inductively implies that for any $t \in [0, \iota_k(\delta)]$, 
$\nabla L_t^{k, \delta}(w) \in \{e_j\}_{j=1}^d$ does not depend on $w$ and is 
only determined by $t$.
Therefore, for any initialization $x \in M$, for all sufficiently small 
$\delta>0$, the gradient flow under $L^{k,\delta}$ for 
$\iota_k(\delta)$ time, i.e., $\phi_{L^{k,\delta}}^{\iota_k(\delta)}(x)$, 
is well-defined. 
Moreover, it follows from \eqref{eq:L_k} that 
\begin{align*}
\int_0^{\iota_{k-1}(\delta)} \nabla L_t^{k, \delta}(w(t)) \diff t
&= \int_0^{\iota_{k-1}(\sqrt{\delta})} \nabla L^{k-1, \sqrt{\delta}}(w(t)) \diff t 
+ \int_{\iota_{k-1}(\sqrt{\delta})}^{\iota_{k-1}(\sqrt{\delta}) + \sqrt{\delta}} 
- e_{j_k} \diff t\\
&\qquad + \int_{\iota_{k-1}(\sqrt{\delta}) + \sqrt{\delta}}^{2\iota_{k-1}(\sqrt{\delta})
+\sqrt{\delta}} -\nabla \overline L^{k-1, \sqrt{\delta}}(w(t)) \diff t
+ \int_{2\iota_{k-1}(\sqrt{\delta}) + \sqrt{\delta}}^{2\iota_{k-1}(\sqrt{\delta})
2\sqrt{\delta}} e_{j_k} \diff t\\
&= \int_0^{\iota_{k-1}(\sqrt{\delta})} \left(\nabla L_t^{k-1, \sqrt{\delta}}(w(t)) -
\nabla \overline{L}_t^{k-1, \sqrt{\delta}}(w(t))\right) \diff t
= 0
\end{align*}
where the last two equalities follow from the fact that $\nabla L_t^{k-1, \sqrt{\delta}}(w)$
does not depend on $w$ and is only determined by $t$ by our construction.

Hence, the mirror flow with respect to the Legendre function $R$ for the time-dependent 
loss $L^{k,\delta}$ will return to the initialization after $\iota_k(\delta)$ 
time since 
\begin{align*}
\nabla R(w(\iota_k(\delta))) - \nabla R(w(0))
= \int_0^{\iota_k(\delta)} -\nabla L^{k, \delta} (w(t)) \diff t = 0.
\end{align*}
This further implies that 
\begin{align*}
G(\xinit) = G\big(\phi_{L^{k, \delta} \circ G}^{\iota_k(\delta)}(\xinit)\big)
\end{align*}
for all sufficiently small $\delta$. 
Then differentiating with $\delta$ on both sides yields 
\begin{align}\label{eq:diff_delta}
\partial G(x) \cdot 
\frac{\diff\phi_{L^{k,\delta}\circ G}^{\iota_k(\delta)}(\xinit)}{\diff \delta}
\bigg\vert_{\delta = 0} =0.
\end{align}
Note that if the following holds:
\begin{align}\label{eq:lie_derivative}
\frac{\diff\phi_{L^{k,\delta}\circ G}^{\iota_k(\delta)}(\xinit)}{\diff \delta}
\bigg\vert_{\delta = 0}
= [[[[\nabla G_{j_1}, \nabla G_{j_2}], \ldots], \nabla G_{j_{k-1}}], 
\nabla G_{j_k}](\xinit),
\end{align}
then combining \eqref{eq:diff_delta} and \eqref{eq:lie_derivative} completes 
the proof, so it remains to verify \eqref{eq:lie_derivative}.

We will prove by induction over $k$, and now let $\{j_i\}_{i=1}^\infty$ be an 
arbitrary sequence where each $j_i \in [d]$.
For notational convenience, we denote for each $k \geq 1$,
\begin{align*}
\pi_{k, \delta}(\cdot) := \phi_{-\ell^{j_k, \delta}}^{\delta}(\cdot)\quad 
\text{and}\quad \Pi_{k, \delta}(\cdot)
:= \phi_{L^{k, \delta}}^{\iota_k(\delta)}(\cdot).
\end{align*}
Then their inverse maps are given by $\pi_{k, \delta}^{-1}(\cdot) 
= \phi_{\ell^{j_k, \delta}}^\delta(\cdot)$ and $\Pi_{k, \delta}^{-1}(\cdot) 
= \phi_{-\overline L^{k, \delta}}^{\iota_k(\delta)}(\cdot)$ respectively.
Since $G$ is smooth, each 
$\Pi_{k, \sqrt{\delta}}$ is a $\cC^\infty$ 
function of $\delta^{1/2^k}$, and we can expand it in $\delta^{1/2^k}$ 
as 
\begin{align}\label{eq:taylor_L}
\Pi_{k, \sqrt{\delta}}(x) 
= x + \sum_{i=1}^{2^k} \frac{\delta^{i / 2^k}}{i!} \Delta_{k, i}(x) 
+ r_{k, \delta}(x)
\end{align}
where the remainder term $r_{k, \delta}(x)$ is continuous in $x$ and for each 
$x \in M$, $r_{k, \delta}(x) = o(\delta)$ (i.e., 
$\lim_{\delta \to 0} \frac{r_{k, \delta}(x)}{\delta}=0$), and each $\Delta_{k, i}$ is defined 
as 
\begin{align*}
\Delta_{k, i}(x) = \frac{\diff^i \Pi_{k, \sqrt{\delta}}(x)}{\diff
(\delta^{1/2^k})^i} \bigg\vert_{\delta=0}.
\end{align*}
In particular, for $k=1$, we have
\begin{align}\label{eq:taylor_L_1}
\Pi_{1, \sqrt{\delta}}(x)
= \pi_{1, \sqrt{\delta}}(x) 
&= x + \sqrt{\delta} \nabla G_{j_1}(x) 
+ \frac{\delta}{2} \partial(\nabla G_{j_1})(x) \nabla G_{j_1}(x) + r_{1, \delta}(x)
\end{align}
where the second equality holds as well for any other $G_j$ in place of $G_{j_1}$, 
with a different but similar remainder term.
For any fixed $K \geq 2$, there is a small open neighborhood of $\xinit$ on $M$, 
denoted by $\cN_{\xinit} \subseteq M$, such that for all $k \in [K]$, 
we have $r_{k, \delta}(x) = o(\delta)$ uniformly over all $x \in \cN_\xinit$, 
so we can replace all $r_{k, \delta}(x)$ by $o(\delta)$ when $x \in \cN_\xinit$. 
Then we claim that for each $k = 2, 3, \ldots, K$, 
\begin{align}\label{eq:induction}
\lim_{\delta \to \infty} \frac{1}{\sqrt{\delta}} 
\sum_{i=1}^{2^{k-1}} \frac{\delta^{i/2^k}}{i!} \Delta_{k, i}(x) 
= [[[\nabla G_{j_1}, \nabla G_{j_2}], \ldots], 
\nabla G_{j_k}](x), \quad \forall x \in \cN_\xinit,
\end{align}
which directly implies \eqref{eq:lie_derivative}.
With a slight abuse of notation, the claim is also true for $k=1$ since
$\Delta_{1, 1}(x) = \nabla G_{j_1}(x)$ by \eqref{eq:taylor_L_1}, so we use this 
as the base case of the induction.
Then, assuming \eqref{eq:induction} holds for $k-1<K$, we proceed to prove it for $k$.
For convenience, further define $\Lie_G(j_{1:k}) = [[[\nabla G_{j_1}, 
\nabla G_{j_2}], \ldots], \nabla G_{j_k}]$.

Combining the Taylor expansion in \eqref{eq:taylor_L} and \eqref{eq:induction}
for $k-1$, we obtain for all $x \in \cN_\xinit$ that
\begin{align*}
\Pi_{k-1, \sqrt{\delta}}(x)
= x + \sqrt{\delta} \cdot \Lie_G(j_{1:(k-1)})(x) + \sum_{i=2^{k-2}+1}^{2^{k-1}}
\frac{\delta^{i/2^{k-1}}}{i!} \Delta_{k-1, i}(x) + o(\delta)
\end{align*}
for sufficiently small $\delta$.
Further apply \eqref{eq:taylor_L_1} with $G_{j_k}$ in place of $G_{j_1}$ for 
sufficiently small $\delta$, and then
\begin{align*}
&\Pi_{k-1, \sqrt{\delta}}\big(\pi_{k, \sqrt{\delta}}(x)\big)\\
&\qquad= \Pi_{k-1, \sqrt{\delta}}\bigg(
x + \sqrt{\delta} \nabla G_{j_k}(x) + \frac{\delta}{2} \partial(\nabla G_{j_k})(x) 
\nabla G_{j_k}(x) + o(\delta)\bigg)\\
&\qquad= x + \sqrt{\delta} \nabla G_{j_k}(x) + \frac{\delta}{2} \partial(\nabla G_{j_k})(x) 
\nabla G_{j_k}(x) + o(\delta)\\
&\qquad\qquad + \sqrt{\delta} \cdot \Lie_G(j_{1:(k-1)})\bigg(x + \sqrt{\delta} 
\nabla G_{j_k}(x) + \frac{\delta}{2} \partial(\nabla G_{j_k})(x) 
\nabla G_{j_k}(x) + o(\delta)\bigg)\\
&\qquad\qquad + \sum_{i=2^{k-2} + 1}^{2^{k-1}} \frac{\delta^{i/2^{k-1}}}{i!} 
\Delta_{k-1, i}\bigg(x + \sqrt{\delta} \nabla G_{j_k}(x) 
+ \frac{\delta}{2} \partial(\nabla G_{j_k})(x) \nabla G_{j_k}(x) + o(\delta)\bigg)\\
&\qquad\qquad + r_{k-1, \delta}\bigg(x + \sqrt{\delta} \nabla G_{j_k}(x) 
+ \frac{\delta}{2} \partial(\nabla G_{j_k})(x) \nabla G_{j_k}(x) + o(\delta)\bigg)
\end{align*}
where the second equality follows from the Taylor expansion of 
$\Pi_{k-1, \sqrt{\delta}}$ and that $\pi_{k, \sqrt{\delta}}(x) \in \cN_\xinit$ 
for sufficiently small $\delta$.
Then by the Taylor expansion of $\Lie_G(j_{1:(k-1)})$ and each $\Delta_{k-1, i}$, 
we have for all $x \in \cN_\xinit$,
\begin{align}\label{eq:induction_1}
\Pi_{k-1, \sqrt{\delta}} \big(\pi_{k, \sqrt{\delta}}(x)\big)
&= x + \sqrt{\delta} \nabla G_{j_k}(x) 
+ \sqrt{\delta} \cdot \Lie_G(j_{1:(k-1)})(x)
+ \frac{\delta}{2} \partial(\nabla G_{j_k})(x) \nabla G_{j_k}(x) \notag\\
&\qquad + \delta \cdot \partial\Lie_G(j_{1:(k-1)})(x) \nabla G_{j_k}(x)
+ \sum_{i=2^{k-2}+1}^{2^{k-1}} \frac{\delta^{i/2^{k-1}}}{i!} \Delta_{k-1,i}(x)
+ o(\delta)
\end{align}
for sufficiently small $\delta$.
For the other way around, we similarly have 
\begin{align}\label{eq:induction_2}
\pi_{k, \sqrt{\delta}}\big(\Pi_{k-1, \sqrt{\delta}}(x)\big)
&= \pi_{k, \sqrt{\delta}}\bigg(
x + \sqrt{\delta} \cdot \Lie_G(j_{1:(k-1)})(x) 
+ \sum_{i=2^{k-2}+1}^{2^{k-1}} \frac{\delta^{i/2^{k-1}}}{i!} \Delta_{k-1, i}(x) 
+ o(\delta)\bigg)\notag\\
&= x + \sqrt{\delta} \nabla G_{j_k}(x) 
+ \sqrt{\delta} \cdot \Lie_G(j_{1:(k-1)}) 
+ \frac{\delta}{2} \partial(\nabla G_{j_k})(x) \nabla G_{j_k}(x)\notag\\
&\qquad + \delta \partial(\nabla G_{j_k})(x) \Lie_G(j_{1:(k-1)})(x) 
+ \sum_{i=2^{k-2}+1}^{2^{k-1}} \frac{\delta^{i/2^k}}{i!} \Delta_{k-1, i}(x) + o(\delta)
\end{align}
for all $x \in \cN_\xinit$, when $\delta$ is sufficiently small.
Note that $x = \pi_{k, \sqrt{\delta}}^{-1} \circ \Pi_{k-1, \sqrt{\delta}}^{-1} 
\circ \Pi_{k-1, \sqrt{\delta}} \circ \pi_{k, \sqrt{\delta}}(x)$, thus
\begin{align}\label{eq:flow_identity}
\Pi_{k, \delta}(x) - x &= \pi_{k, \sqrt{\delta}}^{-1} 
\circ \Pi_{k-1, \sqrt{\delta}}^{-1} \circ \pi_{k, \sqrt{\delta}} 
\circ \Pi_{k-1, \sqrt{\delta}}(x) - x\notag\\
&= \pi_{k, \sqrt{\delta}}^{-1} \circ \Pi_{k-1, \sqrt{\delta}}^{-1} 
\circ \pi_{k, \sqrt{\delta}} \circ \Pi_{k-1, \sqrt{\delta}}(x) 
- \pi_{k, \sqrt{\delta}}^{-1} \circ \Pi_{k, \sqrt{\delta}}^{-1} 
\circ \Pi_{k, \sqrt{\delta}} \circ \pi_{k, \sqrt{\delta}}(x)\notag\\
&= \pi_{k, \sqrt{\delta}}^{-1} \circ \Pi_{k-1, \sqrt{\delta}}^{-1} 
\circ \pi_{k, \sqrt{\delta}} \circ \Pi_{k-1, \sqrt{\delta}}(x)
- \pi_{k, \sqrt{\delta}} \circ \Pi_{k-1, \sqrt{\delta}}(x)\notag\\
&\qquad + \pi_{k, \sqrt{\delta}} \circ \Pi_{k-1, \sqrt{\delta}}(x)
- \Pi_{k, \sqrt{\delta}} \circ \pi_{k, \sqrt{\delta}}(x)\notag\\
&\qquad + \Pi_{k-1, \sqrt{\delta}}(x) \circ \pi_{k, \sqrt{\delta}} -
\pi_{k, \sqrt{\delta}}^{-1} \circ \Pi_{k, \sqrt{\delta}}^{-1} 
\circ \Pi_{k, \sqrt{\delta}} \circ \pi_{k, \sqrt{\delta}}(x)\notag\\
&=  \Pi_{k-1, \sqrt{\delta}} \circ \pi_{k, \sqrt{\delta}}(x) 
- \pi_{k, \sqrt{\delta}} \circ \Pi_{k-1, \sqrt{\delta}}(x)+ o(\delta)
\end{align}
where the last equality follows from the Taylor expansion of 
$\pi_{k, \sqrt{\delta}}^{-1} \circ \Pi_{k-1, \sqrt{\delta}}^{-1}(\cdot)$ in 
terms of $\sqrt{\delta}$.
Now, combining \eqref{eq:induction_1}, \eqref{eq:induction_2} and 
\eqref{eq:flow_identity}, we obtain 
\begin{align}\label{eq:flow_diff}
\Pi_{k, \delta}(x) - x 
&= \delta\left(\partial(\nabla G_{j_k})(x) \Lie_G(j_{1:(k-1)})(x) 
- \partial\Lie_G(j_{1:(k-1)})(x) \nabla G_{j_k}(x)\right) + o(\delta)\notag\\
&= \delta \cdot [\Lie_G(j_{1:(k-1)}), \nabla G_{j_k}](x) + o(\delta)
\end{align}
where the second equality follows from the definition of Lie bracket.
Comparing \eqref{eq:flow_diff} with \eqref{eq:taylor_L} yields \eqref{eq:induction}.
This completes the induction for $k \in [K]$ and hence finishes the proof as 
$K$ is arbitrary.
\end{proof}

\subsection{Convergence for  gradient flow with commuting parametrization}\label{apdx:convergence}

\begin{proof}[Proof of \Cref{thm:convergence_commuting_flow}]
Recall that the dynamics of $w(t)$ is given by
\begin{align*}
    \diff w(t) = - \nabla^2 R(w(t))^{-1} \nabla L(w(t)) \diff t, \qquad w(0) = G(\xinit).
\end{align*}
By \Cref{lem:commuting_potential}, we know that $R$ is a Legendre function.
Therefore, when $R$ is further a Bregman function, we can apply \Cref{thm:convergence_rgf} to obtain the convergence of $w(t)$.
This finishes the proof.
\end{proof}

Based on \Cref{thm:bregman}, we can prove the trajectory convergence of $w(t)$ 
for the special case where $\Omega_w(\xinit;G) = \RR^d$ as summarized in \Cref{cor:convergence_commuting_flow_Rd}.

\begin{proof}[Proof of \Cref{cor:convergence_commuting_flow_Rd}]
It suffices to verify that $R$ is a Bregman function in this case.
By \Cref{lem:commuting_potential}, we know that $R$ is 
a Legendre function and satisfies that $\RR^d = \Omega_w(\xinit;G) = \dom \nabla R \subseteq \dom R \subseteq \RR^d$, which implies $\dom R = \RR^d$.
Moreover, the domain of its convex conjugate $Q$ is also $\RR^d$. 
Then by \Cref{thm:bregman}, we see that $R$ is a Bregman function.
This finishes the proof.
\end{proof}

Next, we prove that for a class of commuting quadratic parametrizations, 
the corresponding Legendre function is also a Bregman function, thus guaranteeing the 
trajectory convergence.

\begin{proof}[Proof of \Cref{cor:convergence_commuting_quadratic_parametrization}]
Since $A_1, A_2, \ldots, A_d$ commute with each other, these matrices can be simultaneously diagonalized.
Thus we can assume without loss of generality that each $A_i = \diag(\lambda_i)$ where $\lambda_i \in \RR^D$, then $G_i(x) = \lambda_i^\top x^{\odot 2}$.
For convenience, we denote $\Lambda = (\lambda_1, \lambda_2, \ldots, \lambda_d)^\top \in \RR^{d\times D}$, so the parametrization is given by $G(x) = \Lambda x^{\odot 2}$.
Note that for each $i \in [d]$, $\nabla G_i(x) = 2 \lambda_i \odot x$ and $\nabla^2 G_i(x) = 2 \diag(\lambda_i)$, so for any $i, j \in [d]$, we have 
\begin{align*} 
[\nabla G_i, \nabla G_j](x) = 4 \diag(\lambda_i) \lambda_j \odot x - 4 \diag(\lambda_j) \lambda_i \odot x = 0.
\end{align*}
Therefore, we see that $G: \RR_+^D \to \RR^d$ is a commuting parametrization. 
Also, for any $t \in \RR$, $x(t) = \xinit - \int_0^t \nabla G_i(x(s))\diff s =  \xinit \odot e^{-2\lambda_i t}$, which proves the first and the second claims. 
Moreover, if the sign of each coordinate of $x$ will not change from that of initialization, (sign means $+$,$-$ or $0$). Without loss of generality, below we will assume every coordinate is non-zero at initialization (otherwise we just ignore it). We can also assume the coordinates at initialization are all positive, as the negatives will induce the same trajectory in terms of $G(x)$.
By \Cref{thm:commuting_to_mirror}, the dynamics of $w(t) = G(x(t))$ is given by
\begin{align*}
    \diff w(t) = -\nabla^2 R(w(t))^{-1} \nabla L(w(t)) \diff t, \qquad w(0) = G(\xinit)
\end{align*}
for some Legendre function $R$ whose conjugate is denoted by $Q$.
To apply the results in \Cref{thm:convergence_commuting_flow}, it suffices to show that this $R$ is a Bregman function.

To do so, we further denote $\widetilde w = x^{\odot 2}$ and $\widetilde G(x) = x^{\odot 2}$, then $w = \Lambda \widetilde w$ and in this case $\widetilde G$ is a commuting parametrization for $\widetilde w$ defined on $M = \RR_+^D$.
Also, we have $\partial G(x) = \Lambda \partial\widetilde G(x)$.
Let $\widetilde L: \RR^d \to \RR$ be defined by $\widetilde L(\widetilde w) = L(\Lambda\widetilde w)$, which satisfies that $\nabla \widetilde L(\widetilde w) = \Lambda^\top \nabla L(\Lambda \widetilde w)$.
Then the gradient flow with parametrization $\widetilde G$ governed by $-\nabla  (\widetilde L \circ \widetilde G)(x)$ is given by
\begin{align*}
    \diff x(t) &= - \nabla  (\widetilde L \circ \widetilde G)(x) \diff t
    = - \partial \widetilde G(x(t))^\top \nabla \widetilde L(\widetilde G(x(t)) \diff t\\
    &= - \partial \widetilde G(x(t))^\top \Lambda^\top \nabla L(\Lambda \widetilde G(x(t)) \diff t\\
    &= -\partial G(x(t))^\top \nabla L(G(x(t)) \diff t,
\end{align*}
which yields the same dynamics of the gradient flow with parametrization $G$ governed by $-\nabla (L \circ G)(x)$.
Therefore, we have $w(t) = G(x(t)) = \Lambda \widetilde G(x(t)) = \Lambda \widetilde w(t)$, where again by \Cref{thm:commuting_to_mirror}, the dynamics of $\widetilde w(t)$ is
\begin{align*}
    \diff \widetilde w(t) = - \nabla^2 \widetilde R(\widetilde w(t))^{-1} \nabla \widetilde L(\widetilde w(t)) \diff t, \qquad \widetilde w(0) = \widetilde G(\xinit)
\end{align*}
for some Legendre function $\tR$ whose conjugate is denoted by $\tQ$.
For any $x \in M$ and $\tmu \in \RR^D$, we define $\tpsi(x;\tmu) = \phi_{\tG_1}^{\tmu_1} \circ \phi_{\tG_2}^{\tmu_2} \circ \cdots \circ \phi_{\tG_D}^{\tmu_D}(x)$.
We need the following lemma.
\begin{lemma}\label{lem:change_space}
In the setting of the proof of \Cref{cor:convergence_commuting_quadratic_parametrization}, 
for any $\mu \in \RR^d$ and $x \in M$, we have $\psi(x;\mu) = \tpsi(x;\Lambda^\top \mu)$.
\end{lemma}

Recall from \Cref{lem:commuting_potential} that $\nabla Q(\mu) = G(\psi(\xinit;\mu))$ for any $\mu \in \RR^d$ and $\nabla \widetilde Q(\widetilde\mu) = \widetilde G(\tpsi(\xinit;\widetilde\mu))$ for any $\widetilde \mu \in \RR^D$.
Note that
\begin{align}\label{eq:Q}
    \nabla Q(\mu) = \Lambda \psi(\xinit;\mu)^{\odot 2} 
    = \Lambda \tpsi(\xinit; \Lambda^\top \mu)^{\odot 2}
    = \Lambda \widetilde G(\tpsi(\xinit; \Lambda^\top \mu))
    = \Lambda \nabla \widetilde Q(\Lambda^\top \mu)
\end{align}
where the second equality follows from \Cref{lem:change_space}.
This implies that $Q(\mu) = \widetilde Q(\Lambda^\top \mu) + C$ for some constant $C$.
Recall the definition of convex conjugate, and we have 
\begin{align*}
    \widetilde R(\widetilde w) = \sup_{\widetilde\mu \in \RR^D} \langle \widetilde\mu, \tw\rangle - \tQ(\tmu), \qquad
    R(w) = \sup_{\mu \in \RR^d} \langle \mu, w\rangle - Q(\mu).
\end{align*}
Then for any $\tw \in \RR^D$, we have 
\begin{align}\label{eq:R}
    R(\Lambda\tw) &= \sup_{\mu \in \RR^d} \langle \mu, \Lambda \tw\rangle - Q(\mu)
    = \sup_{\mu \in \RR^d} \langle \Lambda^\top \mu, \tw\rangle - \tQ(\Lambda^\top \mu) - C\notag\\
    &= \sup_{\tmu \in \Lambda^\top \RR^d} \langle \tmu, \tw\rangle - \tQ(\tmu) - C
    \leq \sup_{\tmu \in \RR^D} \langle \tmu, \tw\rangle - \tQ(\tmu) - C
    = \tR(\tw) - C
\end{align}
Therefore, for any $\tw \in \dom \tR$, it holds that $R(\Lambda \tw) \leq \tR(\tw) - C < \infty$,
so $\Lambda \dom \tR \subseteq \dom R$, where $\Lambda \dom \tR$ 
On the other hand, by \eqref{eq:Q} and \Cref{prop:strictly_convex}, we have 
\begin{align*}
    \dom \nabla R = \range \nabla Q \subseteq \Lambda \range \nabla \tQ = \Lambda \dom \nabla \tR
\end{align*}
and it follows that 
\begin{align*}
    \inter(\dom R) = \dom \nabla R \subseteq \Lambda \dom \nabla \tR = \Lambda\ \inter(\dom \tR).
\end{align*}
Combining the above, we see that $\dom R = \Lambda \dom \tR$.
As discussed in \Cref{sec:intro}, here it is straightforward to verify that $\widetilde R(\widetilde w) = \sum_{i=1}^D \widetilde w_i(\ln \frac{\widetilde w_i}{x_{\text{init},i}^2} - 1)$, which is indeed a Bregman function with domain $\dom \tR = \overline{\RR_+^D}$.
Thus $\dom R = \Lambda \overline{\RR_+^D}$ is also a closed set.
This yields the first condition in \Cref{def:bregman_function}.

Next, we verify the second condition in \Cref{def:bregman_function}. 
For any $\mu \in \RR^d$, we have
\begin{align*}
    \nabla R(G(\psi(\xinit;\mu))) = \nabla R(\nabla Q(\mu)) = \mu
\end{align*}
and 
\begin{align*}
    \nabla\tR(\tG(\psi(\xinit;\mu))) &= \nabla \tR(\tG(\tpsi(\xinit;\Lambda^\top \mu)))
    = \nabla\tR(\nabla \tQ(\Lambda^\top \mu)) = \Lambda^\top \mu.
\end{align*}
Comparing the above two equalities, we get 
\begin{align}\label{eq:nabla_R}
\nabla \tR(\tw) = \Lambda^\top \nabla R(\Lambda \tw)    
\end{align}
for all $\tw \in \RR_+^D$.
Then for any $\tw \in \overline{\RR^D_+}$ and $y = \Lambda \ty \in \inter(\dom R)$, we have 
\begin{align}
    D_R(\Lambda \tw, y) &= R(\Lambda \tw) - R(y) - \langle \nabla R(y), \Lambda \tw - y\rangle \notag\\
    &= R(\Lambda \tw) - R(\Lambda \ty) - \langle \Lambda^\top \nabla R(\Lambda \ty), \tw - \ty\rangle \notag\\
    &= R(\Lambda \tw) - R(\Lambda \ty) - \langle \nabla \tR(\ty), \tw - \ty\rangle \notag\\
    &= R(\Lambda \tw) - R(\Lambda \ty) - \tR(\tw) + \tR(\ty) + D_{\tR}(\tw, \ty)\label{eq:D_R}\\
    &\geq R(\Lambda \tw) - \tR(\tw) + C + D_{\tR}(\tw,\ty)\notag
\end{align}
where the inequality follows from \eqref{eq:R}.
Therefore, we further have for any $\alpha \in \RR$
\begin{align*}
    \{y \in \inter(\dom R) \mid D_R(\Lambda \tw, y) \leq \alpha\} \subseteq \Lambda \{\ty \in \RR_+^D \mid D_{\tR}(\tw, \ty) \leq \alpha - R(\Lambda \tw) + \tR(\tw) - C\}
\end{align*}
where the right-hand side is bounded since $\tR$ is a Bregman function, and so is the left-hand side.

Finally, we verify the third condition in \Cref{def:bregman_function}.
Consider any $w \in \dom R$ and sequence $\{w_i\}_{i=1}^\infty \subset \inter(\dom R)$ such that $\lim_{i\to\infty} w_i = w$. Since $\dom R = \Lambda \dom \tR$, there is some $\tw \in \overline{\RR_+^D}$  such that $w = \Lambda \tw$ and some $\tw_i \in \RR_+^D$ for each $i\in\mathbb{N}^+$ such that $w_i = \Lambda \tw_i$. We have that 
\begin{align*}
    R(w) - R(w_i) &= \int_0^1 \langle \nabla R((1-t)w_i + tw), w - w_i\rangle \diff t\\
    &= \int_0^1 \langle \Lambda^\top \nabla R(\Lambda((1-t) \tw_i + t\tw)), \tw - \tw_i\rangle \diff t\\
    &= \int_0^1 \langle \nabla \tR((1-t) \tw_i + t\tw), \tw - \tw_i\rangle \diff t\\
    &= \tR(\tw) - \tR(\tw_i).
\end{align*}
Combining this with \eqref{eq:D_R}, we get $D_R(w, w_i) = D_{\tR}(\tw, \tw_i)$.
Note that we can always choose each $\tw_i$ properly such that $\lim_{i \to \infty} \tw_i = \tw$.
Then since $\tR$ is a Bregman function, we have 
\begin{align*}
    \lim_{i \to \infty} D_R(w, w_i) = \lim_{i\to\infty} D_{\tR}(\tw, \tw_i) = 0.
\end{align*}

Therefore, we conclude that $R$ is also a Bregman function.
This finishes the proof.
\end{proof}

\begin{proof}[Proof of \Cref{lem:change_space}]
For each $i \in [D]$ and any $t > 0$, we have 
\begin{align*}
    \phi_{G_i}^t(x) &= x + \int_{s=0}^t - \nabla G_i(\phi_{f_i}^s(x)) \diff s 
    = x + \int_{s=0}^t - \sum_{j=1}^D \lambda_{i,j} \nabla \tG_j(\phi_{f_i}^s(x)) \diff s
    = \tpsi(x; t\lambda_i)
\end{align*}
where the last equality follows from \Cref{lem:gf_commuting}.
Therefore, for any $\mu \in \RR^d$, we further have
\begin{align*}
    \psi(x; \mu) &= \phi_{G_1}^{\mu_1} \circ \phi_{G_2}^{\mu_2} 
    \circ \cdots \circ \phi_{G_d}^{\mu_d}(x)\\
    &= \phi_{\tG_1}^{\mu_1 \lambda_{1,1}} \circ \cdots \circ \phi_{\tG_D}^{\mu_1 \lambda_{1,D}}
    \circ \cdots \circ \phi_{\tG_1}^{\mu_d \lambda_{d,1}} \circ \cdots \circ 
    \phi_{\tG_D}^{\mu_d \lambda_{d,D}}(x)\\
    &= \phi_{\tG_1}^{\sum_{i=1}^d \mu_i \lambda_{i,1}} \circ \cdots \circ 
    \phi_{\tG_D}^{\sum_{i=1}^d \mu_i \lambda_{i,D}}(x)\\
    &= \phi_{\tG_1}^{(\Lambda^\top \mu)_1} \circ \cdots \phi_{\tG_D}^{(\Lambda^\top \mu)_D}(x) 
    = \tpsi(x; \Lambda^\top \mu).
\end{align*}
where the third equality follows from the assumption that $\tG$ is a 
commuting parametrization. 
This finishes the proof.
\end{proof}

\subsection{Results for underdetermined linear regression}\label{apdx:linear_regression}
Here we provide the proof for the implicit bias result for the quadratically overparametrized linear model.

\begin{proof}[Proof of \Cref{cor:overparam_linear_model}]
By symmetry, we assume without loss of generality that all coordinates of $\xinit$ are positive.
Note that for $M = \RR_+^D$ with $D = 2d$, $G: M \to \RR^d$ can be written as $G_i(x) = x^\top A_i x$  where each $A_i = e_i e_i^\top - e_{d+i} e_{d+i}^\top$.
Therefore, this parametrization $G$ satisfies the conditions in \Cref{cor:convergence_commuting_quadratic_parametrization}, which then implies the convergence of $w(t)$.

Next, we identify the function $R$ given by \Cref{thm:commuting_to_mirror}.
we have $\psi(\xinit;\mu) = \binom{u_0\odot e^{-2\mu}}{v_0\odot e^{2\mu}}$ and thus
\begin{align*}
G(\psi(\xinit;\mu)) &= u_0^{\odot 2} \odot e^{-4\mu} - v_0^{\odot 2} \odot e^{4\mu}\\
&= (u_0^{\odot 2} + v_0^{\odot 2}) \odot \sinh(4\mu) 
+ (u_0^{\odot 2} - v_0^{\odot 2}) \odot \cosh(4\mu).
\end{align*} 
So $G(\psi(\xinit;\mu))$ is the gradient of $Q(\mu) = \frac{1}{4} (u_0^{\odot 2} 
+ v_0^{\odot 2}) \odot \cosh(4\mu) + \frac{1}{4} (u_0^{\odot 2} - v_0^{\odot 2}) 
\odot \sinh(4\mu) + C$ where $C$ is an arbitrary constant.
Also note that $(\nabla Q(\mu))_i$ only depends on $\mu_i$, then we have 
\begin{align*}
    (\nabla R(w))_i= (\nabla Q(\mu))_i^{-1}(w) 
    &= \frac{1}{4} \ln \bigg(\sqrt{1 + \left(\frac{w_i}{2u_{0,i} v_{0,i}}\right)^2} 
    + \frac{w_i}{2u_{0,i} v_{0,i}}\bigg) + \frac{1}{4} \ln \frac{v_{0,i}}{u_{0,i}}\\
    &= \frac{1}{4} \arcsinh\bigg(\frac{w_i}{2u_{0,i} v_{0,i}}\bigg) 
    + \frac{1}{4} \ln\frac{v_{0,i}}{u_{0,i}}
\end{align*}
which further implies that 
\begin{align*}
    R(w) = \frac{1}{4} \sum_{i=1}^d \bigg(w_i \arcsinh \bigg(\frac{w_i}{2u_{0,i} v_{0,i}}\bigg) 
    - \sqrt{w_i^2 + 4 u^2_{0,i} v^2_{0,i}} - w_i\ln \frac{u_{0,i}}{v_{0,i}}\bigg) + C.
\end{align*}
This finishes the proof.
\end{proof}

\section{Omitted proofs in \Cref{sec:MF_to_CGF}}

We first prove the following intermediate result that will be useful in the proof of \Cref{thm:mirror_to_commuting}.

\begin{lemma}\label{lem:embedding_inverse}
Under the setting of \Cref{thm:mirror_to_commuting}, let $F$ 
be the smooth map that isometrically embeds $(\inter(\dom R), g^R)$ into $(\RR^D, \overline{g})$. 
Let $M = \range(F)$, and denote the inverse of $F$ by $\tG: M \to \RR^d$.
Then for any $w \in \inter(\dom R)$, it holds that
\begin{align*}
\partial F(w) (\partial F(w)^\top \partial F(w))^{-1} = \partial \tG(F(w))^\top
\quad \text{and} \quad
\partial \tG(F(w)) \partial \tG(F(w))^\top = \nabla^2 R(w)^{-1}.
\end{align*}
\end{lemma}
\begin{proof}[Proof of \Cref{lem:embedding_inverse}]
For any $x\in M$ and $v\in T_x(M)$, consider a parametrized curve 
$\{x(t)\}_{t\geq 0}\subset M$ such that $x(0) = x$ and 
$\frac{\diff x(t)}{\diff t}\big|_{t=0} = v$.
Since $x(t) = F(\tG(x(t)))$ for any $t\geq 0$, 
differentiating with respect to $t$ on both sides and evaluating at $t=0$ yield
\begin{align}\label{eq:identity_F_G}
    v = \partial F(\tG(x)) \partial \tG(x) v.
\end{align} 
Now, for any $w \in \inter(\dom R)$, let $x = F(w)$, then for any $v \in T_x(M)$, 
it follows from \eqref{eq:identity_F_G} that 
\begin{align*}
v^\top \partial F(w) 
= v^\top (\partial F(w) \partial \tG(F(w)))^\top \partial F(w) 
= v^\top \partial \tG(F(w))^\top \partial F(w)^\top \partial F(w).
\end{align*}
Note that the span of the column space of $\partial F(w)$ is exactly $T_x(M)$,
so for any $v$ in the orthogonal complement of $T_x(M)$, it holds that 
\begin{align*}
v^\top \partial F(w) = 0 
= v^\top \partial \tG(F(w))^\top \partial F(w)^\top \partial F(w)
\end{align*} 
where the second equality follows from the fact that for any $i\in[d]$, $\nabla \tG_i(x) \in T_x(M)$. 
Therefore, combining the above two cases, we conclude that 
\begin{align*}
\partial F(w) = \partial \tG(F(w))^\top \partial F(w)^\top \partial F(w).
\end{align*}
Since $\partial F(w)^\top \partial F(w) = \nabla^2 R(w)$ is invertible, we then 
get 
\begin{align*}
\partial \tG(F(w))^\top = \partial F(w) (\partial F(w)^\top \partial F(w))^{-1}.
\end{align*}

Next, for any $w \in \inter(\dom R)$, since $\tG(F(w)) = w$, differentiating on both sides 
yields 
\begin{align*}
    \partial \tG(F(w)) \partial F(w) = I_d.
\end{align*}
Therefore, using the identity proved above, we have 
\begin{align*}
    \partial \tG(F(w)) \partial \tG(F(w))^\top &= \partial \tG(F(w)) \partial F(w) (\partial F(w)^\top \partial F(w))^{-1}\\
    &= (\partial F(w)^\top \partial F(w))^{-1} = \nabla^2 R(w)^{-1}.
\end{align*}
This finishes the proof.
\end{proof}

\begin{proof}[Proof of \Cref{thm:mirror_to_commuting}]
By Nash's embedding theorem, there is a smooth map $F: \inter(\dom R) \to \RR^D$ 
that isometrically embeds $(\inter(\dom R), g^R)$ into $(\RR^D, \overline{g})$. 
Denote $M = \range(F)$, \emph{i.e.}, the embedding of $\inter(\dom R)$ in $\RR^D$.
We further denote the inverse of $F$ on $M$ by $\tG: M \to \RR^d$.
Note $(M,\tG)$ is a global atlas for $M$, we have that $T_x(M) = \mathrm{span}(\{\nabla \tG_i(x)\}_{i=1}^d)$ for all $x \in M$. 
This $\tG$ is almost the commuting parametrization that we seek for, except now 
it is only defined on $M$ but not on an open neighborhood of $M$. 
Yet we can extend $\tG$ to an open neighbourhood of $M$ in the following way: 
First by \citet{foote1984regularity}, for each $x \in M$, there is an open 
neighbourhood $U_x$ of $x$ such that projection function $P$ defined by
\begin{align*}
    P(y) = \argmin_{y' \in M} \|y - y'\|_2
\end{align*} 
is smooth in $U_x$. 
Then we define $U = \cup_{x\in M}U_x$, and extend $\tG$ to $U$ by defining
$G(x):= \tG(P(x))$ for all $x \in U$. 
We have $G(x) = \tG(x)$ for all $x \in M$, and we can verify that 
$\partial G \equiv \partial \tG$ on $M$ as well.
For any $v \in T_x(M)$, let $\{\gamma(t)\}_{t\geq 0}$ be a parametrized curve on $M$ 
such that $\gamma(0) = x$ and 
$\frac{\diff \gamma(t)}{\diff t} \big\vert_{t=0} = v$, then for sufficiently small $t$, 
by Taylor expansion we have  
\begin{align*}
\gamma(t) = P(\gamma(t)) &= P(x) + \partial P(x) (\gamma(t) - x) + o(\|\gamma(t) - x\|_2)\\
&= x + \partial P(x) (\gamma(t) - x) + o(\|\gamma(t) - x\|_2)
\end{align*}
which implies that $v = \partial P(x) v$ by letting $t \to 0$.
While for any $v$ in the orthogonal complement of $T_x(M)$, for sufficiently 
small $\delta > 0$, we have $P(x + \delta v)$ is smooth in $\delta$.
Then since $P(x + \delta v) \in M$ for all sufficiently small $\delta$ by its
definition, we have 
\begin{align}\label{eq:P_orthogonal}
\partial P(x) v &= \frac{\diff P(x + \delta v)}{\diff \delta} \bigg\vert_{\delta=0}
= \lim_{\delta \to 0} \frac{P(x + \delta v) - P(x)}{\delta} =: u \in T_x(M).
\end{align}
Note that $\|x + \delta v - P(x + \delta v)\|_2 \leq \|x + \delta v - P(x)\|_2
= \delta \|v\|_2$, and by Taylor expansion, we have 
\begin{align*}
\|x + \delta v - P(x + \delta v)\|_2 
= \|x + \delta v - \delta \partial P(x) v + O(\delta^2)\|_2
= \|x + \delta v - \delta u + O(\delta^2)\|_2
\end{align*}
where $O(\delta^2)$ denotes a term whose norm is bounded by $C \delta^2$ for 
a constant $C > 0$ for all sufficiently small $\delta$, and the second equality 
follows from \eqref{eq:P_orthogonal}.
Then dividing both sides by $\delta$ and letting $\delta \to 0$, we have 
$\|v\|_2 \geq  \|v - u\|_2$.
Since $u$ is orthogonal to $v$, we must have $u = 0$.
As $v$ is arbitrary, we conclude that $\partial P(x)$ is the orthogonal 
projection matrix onto $T_x(M)$.
Then differentiating both sides of $G(x) = \tG(P(x))$ with $x$ yields 
\begin{align}\label{eq:extension_G}
\partial G(x) = \partial \tG(P(x)) \partial P(x) = \partial \tG(x)
\end{align}
where the second equality follows from the fact that $T_x(M) = \mathrm{span}
(\{\nabla \tG_i(x)\}_{i=1}^d)$.
This further implies that the solution of \Cref{eq:gradient_flow_init} satisfies 
$\diff x/\diff t = -\nabla  (L\circ \tG)(x) \in T_x(M)$, and thus $x(t) \in M$ for all 
$t \geq 0$.

Now we consider the mirror flow
\begin{align*}
    \diff w(t) = - \nabla^2 R(w(t))^{-1} \nabla L_t(w(t)) \diff t, \qquad w(0) = \winit.
\end{align*}
Since $\nabla^2 R(w) = \partial F(w)^\top \partial F(w)$ by the fact that $F$ is
an isometric embedding, we further have
\begin{align*}
    \diff w(t) = - \big(\partial F(w(t))^\top \partial F(w(t))\big)^{-1} \nabla L_t(w(t)) \diff t.
\end{align*}
Now define $x(t) = F(w(t))$, and it follows that 
\begin{align*}
    \diff x(t) &= \partial F(w(t)) \diff w(t) 
    = -\partial F(w(t)) (\partial F(w(t))^\top \partial F(w(t)))^{-1} \nabla L_t
    (w(t)) \diff t\\
    &= -\partial G(F(w(t)))^\top \nabla L_t(w(t)) \diff t
    = -\nabla  (L_t \circ G)(x(t)) \diff t
\end{align*}
where the third equality follows from \Cref{lem:embedding_inverse} and \eqref{eq:extension_G}.

Next, we verify that $G$ restricted on $M$, $\tG$, is a commuting and regular 
parametrization.
First, for any $x \in M$, we have 
$\partial \tG(x)^\top = \partial F(\tG(x)) (\partial F(\tG(x))^\top \partial F(\tG(x)))^{-1}$ 
by \Cref{lem:embedding_inverse} and \eqref{eq:extension_G}.
Since $\nabla^2 R(w) = \partial F(w)^\top \partial F(w)$ is of rank $d$ for all $w \in \inter(\dom R)$, 
it follows that $\partial F(w)$ is also of rank $d$ for all $w \in \inter(\dom R)$, 
thus $\partial \tG(x)$ is of rank $d$ for all $x \in M$.
The commutability of $\{\nabla \tG_i\}_{i=1}^d$ follows directly from \Cref{cor:non_commuting}. Here we just need to show $\rank(\Omega_x(x;\tG)) =\rank(M)$. This is because on one hand $\rank(\Omega_x(x;\tG))\ge \rank(\mathrm{span}
(\{\nabla \tG_i(x)\}_{i=1}^d)) = \rank(M)$, and on the other hand, $\rank(\Omega_x(x;\tG))\le \rank(M)$ since $\Omega_x(x;\tG)\subset M$, for any $x\in M$.

Finally, we show that when $R$ is a mirror map, each $\nabla \tG_j$ is a complete vector field on $M$.
For any $\xinit \in M$, consider  loss $L_t(w) = \langle e_j, w\rangle$,
and the corresponding gradient flow is
\begin{align*}
    \diff x(t) = -\nabla (L_t \circ \tG)(x(t)) \diff t 
    = - \partial \tG(x(t))^\top \nabla L_t(\tG(x(t))) \diff t
    = - \nabla \tG_j(x(t)),
\end{align*}
so $x(t) = \phi_{\tG_j}^{t}(\xinit)$ for all $t\geq 0$.
On the other hand, $w(t) = \tG(x(t))$ satisfies that 
\begin{align*}
    \diff w(t) &= \partial \tG(x(t)) \diff x(t) 
    = -\partial \tG(x(t)) \partial \tG(x(t))^\top \nabla L_t(w(t)) \diff t\\
    &= -\nabla^2 R(w(t))^{-1} \nabla L_t(w(t)) \diff t
    = -\nabla^2 R(w(t))^{-1} e_j \diff t
\end{align*}
where the third equality follows from \Cref{lem:embedding_inverse} and \Cref{eq:extension_G}.
Therefore, rewriting the above as a mirror Flow yields 
\begin{align*}
    \diff \nabla R(w(t)) = - e_j \diff t,
\end{align*}
the solution to which exists for all $t \in  \mathbb{R}$ and is given by 
$\nabla R(w(t)) = e_j t$, so $w(t) = (\nabla R)^{-1}(e_j t)$ is defined for all 
$t \in \RR$ as $\nabla R$ is surjective.
This further implies that $x(t) = F(w(t))$ is well-defined for all $t\in \mathbb{R}$, 
hence $\nabla \tG_j$ is a complete vector field.
\end{proof}

\end{document}